\documentclass{article} 
\PassOptionsToPackage{hyphens}{url}
\usepackage{MLSTEFTNO,times}
\usepackage{subcaption}

\usepackage{amsmath,amsfonts,bm}

\def\eqref#1{equation~\ref{#1}}

\def\1{\bm{1}}

\DeclareMathAlphabet{\mathsfit}{\encodingdefault}{\sfdefault}{m}{sl}
\SetMathAlphabet{\mathsfit}{bold}{\encodingdefault}{\sfdefault}{bx}{n}

\renewcommand{\eqref}[1]{(\ref{#1})}
\usepackage{hyperref}
\hypersetup{hidelinks}
\usepackage{url}
\usepackage{graphicx}
\usepackage{amsthm}
\usepackage{cleveref}
\crefname{equation}{Eq.}{Eqs.}
\usepackage{float}
\usepackage{xurl}
\theoremstyle{definition}
\newtheorem{proposition}{Proposition}[section]

\newtheorem{corollary}[proposition]{Corollary}

\title{Loop Corrections in Random Feature Models:\\
Training Error and Generalization Gap}
\iclrfinalcopy

\author{Taeyoung Kim  \\
Center for AI and Natural Sciences\\
Korea Institute for Advanced Study\\
Seoul 02455, South Korea \\
\texttt{taeyoungkim@kias.re.kr} \\
}

\begin{document}

\maketitle

\begin{abstract}
We study fixed-design random feature ridge regression beyond the mean-kernel approximation.
The expectation is taken over the frozen-feature ensemble, conditional on the training sample.
Because the predictor is a nonlinear function of the empirical kernel, its mean training error, test error, and conditional generalization gap depend on centered kernel covariances as well as on mean kernel objects.
We derive the covariance-level, or one-loop, corrections from a finite resolvent identity.
This avoids an almost-sure Neumann-series assumption and yields an explicit remainder bound for the training error.
The test correction additionally requires mixed train--test covariance tensors and therefore cannot, in general, be recovered from the train-restricted kernel law alone.
Numerical experiments show that the covariance term removes the leading inverse-width discrepancy in a controlled regularization regime, while an ablation that discards the mixed train--test covariance fails for the test error even though the training correction is unchanged.
They also identify a width--regularization boundary, measured by a resolvent-fluctuation diagnostic, beyond which the second-order truncation loses accuracy.
All inverse-width statements are for fixed sample size, depth, design, and positive kernel-level regularization.
The analysis concerns frozen features and does not model feature learning or neural-tangent-kernel evolution.
\end{abstract}

\vspace{0.5em}

\section{Introduction}

\subsection{Random Feature Models}

Random feature models provide a mathematically tractable framework for studying the training and generalization properties of neural networks. In these models, a nonlinear feature map is first generated by randomly sampling the parameters of a neural network from a prescribed initialization ensemble and then freezing them, while only the readout layer is optimized. This setting is closely related to the literature on random kitchen sinks, randomized kernel training, and extreme learning machines \citep{RahimiRecht2007,RahimiRecht2009,Rudi2017,Huang2006,Wang2022}.

A central object in the analysis of random feature models is the kernel induced by the frozen random features. In the infinite-width limit, randomly initialized networks admit deterministic Gaussian-process kernels \citep{Lee2018DeepGP}; related neural tangent kernel limits describe fully trained wide networks \citep{Jacot2018,Lee2019}. For frozen features, the limiting predictor is a kernel ridge regressor, and random-feature approximation results quantify its statistical behavior \citep{Rudi2017}. High-dimensional kernel regression can exhibit additional spectral regimes \citep{Adlam2020}.

At finite width, the induced kernel fluctuates around its ensemble mean. Finite-width corrections to network priors, output distributions, and Bayesian predictions have been developed through Edgeworth, diagrammatic, and layerwise perturbative methods \citep{Antognini2019,Yaida2020,Hanin2024}. Parallel work studies finite-width neural tangent kernels and their evolution under training \citep{DyerGurAri2020,HaninNica2020,HuangYau2020,SeleznovaKutyniok2022,GuillenMisofGerken2026}. The present problem is different but structurally related: the frozen-feature readout depends nonlinearly on the empirical kernel, so ensemble-averaged errors are not determined by its mean alone. This raises the question of how the centered kernel fluctuations enter fixed-design prediction errors.

\subsection{An Effective-Field-Theoretic Perspective on Finite-Width Fluctuations}

In this paper, we approach this problem from the viewpoint of statistical physics and effective field theory. From this perspective, the frozen random features are the fluctuating degrees of freedom, and the relevant observables---such as the training error, test error, and generalization gap---should be understood through their ensemble statistics. The standard mean-kernel approximation then plays the role of a leading-order, or tree-level, description, while finite-width effects appear as corrections generated by fluctuations around that mean \citep{Roberts22}.

This viewpoint is useful because the predictor depends on the inverse of a regularized empirical kernel matrix. Consequently, even small kernel fluctuations enter observables through nonlinear resolvent products. A systematic treatment therefore requires more than replacing the empirical kernel by its ensemble mean: at the first nontrivial order one must retain its covariance, and at subsequent orders one must retain higher centered moments. Following the terminology of neural-network effective theory, we call the covariance-level contribution the one-loop correction \citep{Roberts22}. Here the name is only bookkeeping for the centered second-order term; no path-integral representation is assumed.

For fixed depth and a fixed finite set of inputs, connected correlation functions of wide random networks often admit inverse-width power counting \citep{Hanin2024,HaninNica2020,Roberts22}. Motivated by this power counting, we analyze the leading centered correction to the training error, test error, and conditional generalization gap. The test quantities require not only the covariance of the train--train kernel, but also mixed covariances coupling train--train and test--train kernel objects. This distinction is invisible in a train-restricted mean-kernel analysis.

\subsection{Relation to Finite-Width Network Expansions and Scope}

The works above primarily propagate finite-width statistics through the network or through parameter training. Yaida's layerwise expansion, for example, computes non-Gaussian corrections to network priors and Bayesian inference, while Dyer--Gur-Ari, Huang--Yau, and Guillen--Misof--Gerken analyze derivative correlators, neural tangent kernels, and finite-width training dynamics \citep{Yaida2020,DyerGurAri2020,HuangYau2020,GuillenMisofGerken2026}. Those results address a harder upstream problem than the frozen-feature calculation considered here.

Our contribution is downstream and complementary. We take the joint law of the empirical train--train and test--train kernel objects as the input and determine which parts of that law are visible to ridge-regression risks. Once these fluctuations are specified, the algebraic mechanism is indeed a resolvent expansion. The additional content lies in the observable-level structure: the test error contains mixed train--test covariance maps that are absent from the training error, and Proposition~\ref{prop:nonidentifiability} shows that no train-restricted kernel statistic can recover them universally. The numerical ablation in Subsection~\ref{sec:mixed-ablation} shows that these mixed terms are not merely formal; they can be larger than the final test correction because they cancel against train-restricted contributions. A second consequence is that inverse width alone does not determine perturbative accuracy: the relevant quantity is the resolvent-amplified fluctuation \(q_\Theta=\|(G+\gamma I)^{-1}\Delta_K\|_2\), whose width--regularization boundary is measured in Subsection~\ref{sec:validity-map}.

The frozen-feature restriction is therefore deliberate: it isolates representation sampling from feature learning and permits a conditional, fixed-design statement with a controlled remainder. It also marks a strict limitation. The present results neither describe the evolution of the NTK nor explain representation learning; extending the mixed-covariance observable analysis to trained features remains a separate problem.

\subsection{Contributions}

The main contributions of this paper are as follows.

\begin{itemize}
    \item We derive centered second-order expansions around the exact finite-width mean kernel objects. This convention matches the Monte Carlo baseline used in the experiments and is distinguished explicitly from a full expansion around the infinite-width kernel.
    \item For the training error, we replace the formal infinite Neumann series by a finite resolvent identity and give a nonasymptotic operator-norm bound on the discarded terms. Under fixed-design moment assumptions, the remainder is \(O(n^{-3/2})\).
    \item We give explicit covariance corrections for the test error and the conditional generalization gap. These formulas expose mixed train--test tensors that are not identifiable from train-restricted spectral data, even when the complete training-kernel law is known.
    \item We test the hierarchy rather than only the correction itself. In a controlled regime, the covariance term removes the leading \(O(n^{-1})\) discrepancy; a mixed-covariance ablation fails for test error under three population shifts; and a two-parameter sweep links truncation accuracy to \(q_\Theta\). These experiments identify both a qualitative train--test distinction and the boundary of the approximation.

\end{itemize}

\section{Preliminaries}

\subsection{Model and problem setting}
Let \(D_N=\{(x_j,y_j)\}_{j=1}^N\) be a fixed training sample, with
\(y_j=y(x_j)\) for a deterministic target \(y:\mathcal X\to\mathbb R\).
All expectations below are conditional on \(D_N\) and are taken only over the
frozen-feature draw \(\Theta\), unless another expectation is displayed.

For \(n_0=\dim(\mathcal X)\), define a depth-\(L\) feature network by
\begin{equation}
\begin{split}
h^{(0)}(x)&=x,\\
z^{(\ell)}(x)&=W_\ell h^{(\ell-1)}(x)+b_\ell,\\
h^{(\ell)}(x)&=\sigma_\ell\!\left(z^{(\ell)}(x)\right),
\qquad \ell=1,\ldots,L,\\
\phi(x;\Theta)&=h^{(L)}(x)\in\mathbb R^n,
\qquad n:=n_L .
\end{split}
\label{random feature}
\end{equation}
The final activation may be set to the identity to recover preactivation
features. Under this convention, an \(L=1\) experiment with a named
activation applies that activation to the feature vector.

The weights and biases are mutually independent across layers, neurons, and
parameter types, with
\begin{equation}
(W_\ell)_{ij}\sim
\mathcal N\!\left(0,\frac{C_W^{(\ell)}}{n_{\ell-1}}\right),
\qquad
(b_\ell)_i\sim\mathcal N\!\left(0,C_b^{(\ell)}\right).
\label{ensemble}
\end{equation}
The Gaussian assumption is used only to connect the empirical-kernel power
counting with the wide-network literature. The resolvent expansions below
apply to any frozen-feature ensemble satisfying the stated moment assumptions.

The predictor is linear in the frozen features,
\begin{equation}
f(x;\Theta,w):=\phi(x;\Theta)^\top w.
\label{model}
\end{equation}
With \(\Phi_{ji}:=\phi_i(x_j;\Theta)\),
\(\Phi\in\mathbb R^{N\times n}\), and
\(\mathbf y=(y_1,\ldots,y_N)^\top\), ridge regression gives
\begin{equation}
\mathbf w^*(\Theta;D_N)
=\arg\min_{\mathbf w\in\mathbb R^n}
\left\{\frac1N\|\Phi\mathbf w-\mathbf y\|^2
+\lambda\|\mathbf w\|^2\right\}
=(\Phi^\top\Phi+N\lambda I_n)^{-1}\Phi^\top\mathbf y.
\label{mini}
\end{equation}
For a population input distribution \(\mu\), the conditional test error is
\begin{equation}
\mathcal E_{\mathrm{test}}(\Theta;D_N)
:=\int_{\mathcal X}
\bigl(f(x;\Theta,\mathbf w^*)-y(x)\bigr)^2\,d\mu(x).
\end{equation}

\subsection{Ensemble averaging and the origin of the correction}

The training error, test error, and conditional generalization gap are random
through \(\Theta\). We study their feature-ensemble averages with \(D_N\)
held fixed. Replacing the empirical kernel \(K_\Theta\) by its exact
finite-width mean \(G=\mathbb E_\Theta[ K_\Theta]\) defines the
\emph{mean-object baseline}. It is not, in general, identical to a baseline
built from the infinite-width kernel, because \(G\) may itself depend on \(n\).
The distinction matters because the observables depend nonlinearly on
\(K_\Theta\) through the resolvent
\begin{equation}
(K_\Theta+\gamma I_N)^{-1},
\qquad
\gamma:=\frac{N\lambda}{n}.
\end{equation}
The centered second-order term depends on kernel covariances. We refer to it
as the one-loop correction, while retaining the term mean-object to make clear
which finite-width contributions are already included in the baseline.

\section{Effective Formulation and Loop Expansion}

We first separate two covariance objects that are easily conflated in a naive
effective-theory treatment: the connected preactivation vertex and the
covariance of the empirical random-feature kernel. The error expansions depend
on the latter. The following standard relation is included only to motivate the
fixed-width power counting.
\begin{proposition}[Connected preactivation vertex under Gaussian initialization]
\label{prop:leading-nongaussian-building-block}
Let \(z^{(\ell)}_{i;\alpha}\) denote the preactivation of neuron \(i\) at layer \(\ell\) evaluated on the sample \(x_\alpha\in D\), and define the stochastic kernel
\begin{equation}
\hat G^{(\ell)}_{\alpha_1\alpha_2}
:=
C_b^{(\ell)}
+
C_W^{(\ell)}
\frac{1}{n_{\ell-1}}
\sum_{j=1}^{n_{\ell-1}}
h^{(\ell-1)}_{j;\alpha_1}h^{(\ell-1)}_{j;\alpha_2},
\end{equation}
with mean kernel
\begin{equation}
G^{(\ell)}_{\alpha_1\alpha_2}
:=
\mathbb E\!\left[\hat G^{(\ell)}_{\alpha_1\alpha_2}\right]
=
\mathbb E\!\left[z^{(\ell)}_{i;\alpha_1}z^{(\ell)}_{i;\alpha_2}\right]. \label{mean-metric}
\end{equation}
Let
\begin{equation}
\Delta \hat G^{(\ell)}_{\alpha_1\alpha_2}
:=
\hat G^{(\ell)}_{\alpha_1\alpha_2}
-
G^{(\ell)}_{\alpha_1\alpha_2},
\end{equation}
and define the connected conditional-kernel vertex by
\begin{equation}
\frac{1}{n_{\ell-1}}
V_{\mathrm c,(\alpha_1\alpha_2)(\alpha_3\alpha_4)}^{(\ell)}
:=
\mathbb E\!\left[
\Delta \hat G^{(\ell)}_{\alpha_1\alpha_2}
\Delta \hat G^{(\ell)}_{\alpha_3\alpha_4}
\right].
\end{equation}
Under the independent centered Gaussian initialization in
\eqref{ensemble}, the connected four-point correlator, with the subscript
\({\mathrm c}\) denoting the joint cumulant, has the form
\begin{equation}
\begin{split}
&\mathbb E\!\left[
z^{(\ell)}_{i_1;\alpha_1}
z^{(\ell)}_{i_2;\alpha_2}
z^{(\ell)}_{i_3;\alpha_3}
z^{(\ell)}_{i_4;\alpha_4}
\right]_{\mathrm c}
\\
&\qquad=
\frac{1}{n_{\ell-1}}
\Big[
\delta_{i_1 i_2}\delta_{i_3 i_4}\,
V_{\mathrm c,(\alpha_1\alpha_2)(\alpha_3\alpha_4)}^{(\ell)}
+
\delta_{i_1 i_3}\delta_{i_2 i_4}\,
V_{\mathrm c,(\alpha_1\alpha_3)(\alpha_2\alpha_4)}^{(\ell)}
\\
&\qquad\qquad\qquad
+
\delta_{i_1 i_4}\delta_{i_2 i_3}\,
V_{\mathrm c,(\alpha_1\alpha_4)(\alpha_2\alpha_3)}^{(\ell)}
\Big].
\end{split}
\label{eq:connected-4pt-vertex}
\end{equation}
\end{proposition}

This is the connected-preactivation relation adapted from the Gaussian
initialization analysis of \citet{Roberts22}. It does not identify
\(V_{\mathrm c}^{(\ell)}\) with the covariance of the empirical feature
kernel defined below.

For notational simplicity, in the sequel we assume that all hidden-layer widths
and the output feature dimension are equal to \(n\), namely
\begin{equation}
n_1=\cdots=n_{L}=n.
\end{equation}
Accordingly, \(n\) denotes the common width parameter controlling the
finite-width expansion.

The proposition concerns preactivations. Error observables instead depend on
the empirical feature kernel, whose covariance contains both connected and
within-feature sampling contributions.

For a fixed feature draw, define
\begin{equation}
\begin{split}
K_\Theta&:=\frac1n\Phi\Phi^\top,\\
k_\Theta(x)&:=\frac1n\Phi\phi(x;\Theta).
\end{split}
\end{equation}
The predictor is
\begin{equation}
f_{\Theta,D_N}(x)
=k_\Theta(x)^\top(K_\Theta+\gamma I_N)^{-1}\mathbf y,
\qquad \gamma:=\frac{N\lambda}{n}.
\end{equation}

Write the exact finite-width mean and centered fluctuation as
\begin{equation}
K_\Theta=G+\Delta_K,
\qquad
G:=\mathbb E_\Theta[K_\Theta].
\end{equation}
The empirical-kernel covariance vertex is defined, at each \(n\), by
\begin{equation}
V^{(n)}_{(\alpha\beta)(\gamma\delta)}
:=
n\,\mathbb E_\Theta
\bigl[(\Delta_K)_{\alpha\beta}(\Delta_K)_{\gamma\delta}\bigr].
\label{eq:KV-identification}
\end{equation}
We suppress the superscript \(n\) below. This \(V\) is defined by the
empirical-kernel covariance and is not the connected preactivation vertex
\(V_{\mathrm c}^{(L)}\). For conditionally Gaussian preactivation features,
the law of total covariance gives, schematically,
\begin{equation}
V_{(\alpha\beta)(\gamma\delta)}
=G_{\alpha\gamma}G_{\beta\delta}
+G_{\alpha\delta}G_{\beta\gamma}
+V_{\mathrm c,(\alpha\beta)(\gamma\delta)}^{(L)}
+O(n^{-1}).
\label{eq:empirical-connected-relation}
\end{equation}
The first two terms are the within-feature Wick contributions. In particular,
the empirical kernel fluctuates even when the connected conditional-kernel
vertex vanishes. Equation~\eqref{eq:empirical-connected-relation} concerns
preactivation features, equivalently an identity final activation; for a
nonlinear final activation, \(V\) remains defined by
\eqref{eq:KV-identification}, but no Gaussian Wick reduction is assumed.

For \(A\in\mathbb R^{N\times N}\), define
\begin{equation}
(V\star A)_{\alpha\delta}
:=
\sum_{\beta,\gamma=1}^N
V_{(\alpha\beta)(\gamma\delta)}A_{\beta\gamma}.
\label{eq:V-star-A}
\end{equation}
Then
\begin{equation}
\mathbb E_\Theta[\Delta_K A\Delta_K]
=
\frac{1}{n}(V\star A).
\label{eq:delta-A-delta}
\end{equation}

Throughout Sections~3--4, \(D_N\), \(N\), and \(L\) are fixed,
\(\gamma\geq\gamma_0>0\), and \(n\to\infty\). Thus a fixed-\(\gamma\)
width limit uses \(\lambda_n=\gamma n/N\). We assume
\begin{equation}
\mathbb E_\Theta[\|\Delta_K\|_2^p]
\leq c_p n^{-p/2},
\qquad p=2,\ldots,6,
\label{eq:kernel-moment-assumption}
\end{equation}
with constants allowed to depend on \(D_N,N,L,\gamma_0\). Analogous
integrable mixed-moment assumptions are stated when test points are
introduced. No claim is made that the bounds are uniform in depth, sample
size, or vanishing regularization. In particular, depth-dependent vertices
can produce an effective \(L/n\) correction for scale-invariant activations
\citep{Roberts22,Hanin2024}.

Because \(G\) is the exact finite-\(n\) mean, the expansion below isolates
centered covariance effects. A complete expansion relative to an
infinite-width kernel would additionally require the finite-width shifts of
all mean objects.

\subsection{Expansion of the training error}

We first derive a perturbative expansion of the training error around the
mean kernel.
For a fixed parameter \(\Theta\), the ridge solution yields
\begin{equation}
\begin{split}
   & \mathbf w^*(\Theta;D_N)
\\&=
\Phi(\Theta)^{\top}
\bigl(\Phi(\Theta)\Phi(\Theta)^{\top}+N\lambda I_N\bigr)^{-1}\mathbf y
\\&=
\frac1n\Phi(\Theta)^\top
\bigl(K_\Theta+\gamma I_N\bigr)^{-1}\mathbf y,
\qquad
\gamma:=\frac{N\lambda}{n}.
\end{split}
\end{equation}
and hence
\begin{equation}
f_{\Theta,D_N}(X)
=
K_\Theta (K_\Theta+\frac{N\lambda}{n} I_N)^{-1}\mathbf y,
\qquad
K_\Theta:=\frac{1}{n}\Phi(\Theta)\Phi(\Theta)^{\top}.
\end{equation}
Accordingly, the training error is given by
\begin{equation}
\mathcal{E}_{\text{tr}}(\Theta)
:=
\frac{1}{N}\|\Phi(\Theta)\mathbf w^*(\Theta;D_N)-\mathbf y\|^2.
\end{equation}

\begin{proposition}[Controlled centered expansion of the training error]
\label{prop:loop-expansion-training-error}
Let
\begin{equation}
K_\Theta=G+\Delta_K,
\qquad
\mathbb E_\Theta[\Delta_K]=0,
\end{equation}
and define
\begin{equation}
R:=(G+\gamma I_N)^{-1}.
\end{equation}
Then the expected training error admits the expansion
\begin{equation}
\begin{split}
\mathbb E_\Theta[\mathcal E_{\mathrm{tr}}]
&=
\frac{\gamma^2}{N}\mathbf y^\top
\left(R^2+L_{\mathrm{tr}}+\mathcal R_{\mathrm{tr}}\right)\mathbf y,
\end{split}
\end{equation}
where
\begin{equation}
L_{\mathrm{tr}}
=
R^2\mathbb E_\Theta[\Delta_KR\Delta_K]R
+R\mathbb E_\Theta[\Delta_KR^2\Delta_K]R
+R\mathbb E_\Theta[\Delta_KR\Delta_K]R^2.
\end{equation}
For \(\delta_\Theta:=\|\Delta_K\|_2\), the remainder satisfies
\begin{equation}
\|\mathcal R_{\mathrm{tr}}\|_2
\leq
\mathbb E_\Theta\!\left[
\frac{4\delta_\Theta^3}{\gamma^5}
+\frac{3\delta_\Theta^4}{\gamma^6}
+\frac{2\delta_\Theta^5}{\gamma^7}
+\frac{\delta_\Theta^6}{\gamma^8}
\right].
\label{eq:training-remainder-bound}
\end{equation}
Consequently, \(\|\mathcal R_{\mathrm{tr}}\|_2=O(n^{-3/2})\) under
\eqref{eq:kernel-moment-assumption}. No realization-wise condition
\(\|R\Delta_K\|_2<1\) is required.
\end{proposition}

\begin{proof}
Starting from the definition of the training error, we have
\begin{equation}
\begin{split}
\mathcal{E}_{\text{tr}}(\Theta)
&=
\frac{1}{N}
\|\Phi(\Theta)\mathbf w^*(\Theta;D_N)-\mathbf y\|^2 \\
&=
\frac{1}{N}
\Bigl(
\mathbf w^{* \top}\Phi(\Theta)^\top\Phi(\Theta)\mathbf w^*
-2\mathbf y^\top \Phi(\Theta)\mathbf w^*
+\|\mathbf y\|^2
\Bigr).
\end{split}
\end{equation}
Substituting
\begin{equation}
\mathbf w^*
=
\frac1n\Phi(\Theta)^\top
(K_\Theta+\gamma I_N)^{-1}\mathbf y,
\end{equation}
we obtain
\begin{equation}
\begin{split}
\mathcal{E}_{\text{tr}}(\Theta)
=
\frac{1}{N}
\Bigl[
&\mathbf y^\top
K_\Theta (K_\Theta+\gamma I_N)^{-1}
K_\Theta (K_\Theta+\gamma I_N)^{-1}\mathbf y \\
&-2\mathbf y^\top K_\Theta(K_\Theta+\gamma I_N)^{-1}\mathbf y
+\|\mathbf y\|^2
\Bigr].
\end{split}
\end{equation}
Using the identity
\begin{equation}
K_\Theta(K_\Theta+\gamma I_N)^{-1}
=
I_N-\gamma (K_\Theta+\gamma I_N)^{-1},
\end{equation}
the above expression simplifies to
\begin{equation}
\mathcal{E}_{\text{tr}}(\Theta)
=
\frac{\gamma^2}{N}\,
\mathbf y^\top (K_\Theta+\gamma I_N)^{-2}\mathbf y.
\end{equation}
Since \(\gamma=N\lambda/n\), equivalently
\begin{equation}
\mathcal{E}_{\text{tr}}(\Theta)
=
\frac{N\lambda^2}{n^2}\,
\mathbf y^\top (K_\Theta+\gamma I_N)^{-2}\mathbf y.
\end{equation}
Taking expectation over \(\Theta\), we get
\begin{equation}
\mathbb E_\Theta[\mathcal{E}_{\text{tr}}]
=
\frac{\gamma^2}{N}\,
\mathbf y^\top
\mathbb E_\Theta\big[(K_\Theta+\gamma I_N)^{-2}\big]
\mathbf y.
\end{equation}

Set \(R_\Theta=(K_\Theta+\gamma I_N)^{-1}\). Repeated use of the
resolvent identity gives the finite, exact formula
\begin{equation}
R_\Theta
=R-R\Delta_KR+R\Delta_KR\Delta_KR
-R\Delta_KR\Delta_KR\Delta_KR_\Theta.
\label{eq:finite-resolvent-identity}
\end{equation}
Write the four terms on the right as \(A_0,A_1,A_2,E_3\), respectively,
and let \(Q=A_0+A_1+A_2\). Since both \(G\) and \(K_\Theta\) are positive
semidefinite,
\begin{equation}
\|R\|_2\leq\gamma^{-1},
\qquad
\|R_\Theta\|_2\leq\gamma^{-1}.
\end{equation}
For \(\delta_\Theta=\|\Delta_K\|_2\), it follows that
\begin{equation}
\|A_j\|_2\leq\gamma^{-(j+1)}\delta_\Theta^j,
\quad j=0,1,2,
\qquad
\|E_3\|_2\leq\gamma^{-4}\delta_\Theta^3.
\end{equation}
In \(R_\Theta^2=(Q+E_3)^2\), the terms of total degree at most two
are
\begin{equation}
A_0^2+A_0A_1+A_1A_0+A_0A_2+A_2A_0+A_1^2.
\end{equation}
The expectation of the linear terms vanishes because
\(\mathbb E_\Theta\Delta_K=0\); the quadratic terms are precisely
\(L_{\mathrm{tr}}\). The remaining terms are
\begin{equation}
A_1A_2+A_2A_1+A_2^2+QE_3+E_3Q+E_3^2.
\end{equation}
Submultiplicativity and the preceding bounds give
\begin{equation}
\begin{split}
\|A_1A_2+A_2A_1+A_2^2+QE_3+E_3Q+E_3^2\|_2
\leq{}&
\frac{4\delta_\Theta^3}{\gamma^5}
+\frac{3\delta_\Theta^4}{\gamma^6}\\
&+\frac{2\delta_\Theta^5}{\gamma^7}
+\frac{\delta_\Theta^6}{\gamma^8}.
\end{split}
\end{equation}
Taking expectations proves both the expansion and
\eqref{eq:training-remainder-bound}.
\end{proof}

\begin{corollary}[Training correction in terms of \(G\) and \(V\)]
\label{cor:training-error-GV}
Under \eqref{eq:KV-identification}, let
\begin{equation}
R:=(G+\gamma I_N)^{-1}.
\end{equation}
Then the covariance-level correction is
\begin{equation}
\begin{split}
L_{\mathrm{tr}}
=
\frac{1}{n}\Big(
&R^2(V\star R)R
+R(V\star R^2)R
+R(V\star R)R^2
\Big).
\end{split}
\label{eq:Ttr2-GV}
\end{equation}
Define
\begin{equation}
\mathcal E_{\mathrm{tr}}^{(\mathrm{mean})}
:=\frac{\gamma^2}{N}\mathbf y^\top R^2\mathbf y,
\qquad
\mathcal E_{\mathrm{tr}}^{(\mathrm{cov})}
:=\frac{\gamma^2}{N}\mathbf y^\top L_{\mathrm{tr}}\mathbf y.
\end{equation}
Therefore,
\begin{equation}
\mathbb E_\Theta[\mathcal E_{\mathrm{tr}}]
=
\frac{\gamma^2}{N}\,
\mathbf y^\top
\left[
R^2
+
\frac{1}{n}\Big(
R^2(V\star R)R
+R(V\star R^2)R
+R(V\star R)R^2
\Big)
\right]\mathbf y
+
O\Big(n^{-3/2}\Big).
\end{equation}
For fixed \(D_N,N,L,\gamma\), the displayed centered correction is
\(O(n^{-1})\). This statement is not uniform in \(N\), depth, or
\(\gamma\).
\end{corollary}

\subsection{Expansion of the test error}

We now derive the loop expansion for the test error.
Using the kernel representation of the predictor,
\begin{equation}
f_{\Theta,D_N}(x)
:=
k_\Theta(x)^\top (K_\Theta+\frac{N\lambda}{n} I_N)^{-1}\mathbf y,
\end{equation}
the test error is written as
\begin{equation}
\mathcal E_{\mathrm{test}}(\Theta;D_N)
=
\int_{\mathcal X}
\Big(
k_\Theta(x)^\top (K_\Theta+\gamma I_N)^{-1}\mathbf y
-y(x)
\Big)^2\,d\mu(x).
\end{equation}

To simplify the notation, define
\begin{equation}
\gamma:=\frac{N\lambda}{n},
\qquad
R_\Theta:=(K_\Theta+\gamma I_N)^{-1},
\end{equation}
and introduce
\begin{equation}
C_\Theta
:=
\int_{\mathcal X}
k_\Theta(x)k_\Theta(x)^\top\,d\mu(x),
\end{equation}
\begin{equation}
\mathbf b_\Theta
:=
\int_{\mathcal X}
y(x)\,k_\Theta(x)\,d\mu(x),
\qquad
c
:=
\int_{\mathcal X} y(x)^2\,d\mu(x).
\end{equation}
Then the test error admits the quadratic form
\begin{equation}
\mathcal E_{\mathrm{test}}(\Theta;D_N)
=
\mathbf y^\top R_\Theta C_\Theta R_\Theta \mathbf y
-2\mathbf b_\Theta^\top R_\Theta \mathbf y
+c.
\end{equation}

\begin{proposition}[Centered expansion of the test error]
\label{prop:loop-expansion-test-error}
Let
\begin{equation}
K_\Theta=G+\Delta_K,
\qquad
C_\Theta=\bar C+\Delta_C,
\qquad
\mathbf b_\Theta=\bar{\mathbf b}+\Delta_b,
\end{equation}
where
\begin{equation}
G:=\mathbb E_\Theta[K_\Theta],
\qquad
\bar C:=\mathbb E_\Theta[C_\Theta],
\qquad
\bar{\mathbf b}:=\mathbb E_\Theta[\mathbf b_\Theta],
\end{equation}
and assume
\begin{equation}
\mathbb E_\Theta[\Delta_K]=0,
\qquad
\mathbb E_\Theta[\Delta_C]=0,
\qquad
\mathbb E_\Theta[\Delta_b]=0.
\end{equation}
Define further
\begin{equation}
R:=(G+\gamma I_N)^{-1}.
\end{equation}
Assume that \(\|\bar C\|_2\) and \(\|\bar{\mathbf b}\|_2\) remain
bounded and that the joint centered products of \(\Delta_K,\Delta_C,\Delta_b\)
appearing at total degree \(q=2,\ldots,6\) have integrable operator-norm moments of
order \(O(n^{-q/2})\). Then the expected test error admits the expansion
\begin{equation}
\mathbb E_\Theta[\mathcal E_{\mathrm{test}}]
=
\mathcal E_{\mathrm{test}}^{(\mathrm{mean})}
+
\mathcal E_{\mathrm{test}}^{(\mathrm{cov})}
+
\mathcal R_{\mathrm{test}},
\qquad
|\mathcal R_{\mathrm{test}}|=O(n^{-3/2}),
\end{equation}
where the mean-object baseline is
\begin{equation}
\mathcal E_{\mathrm{test}}^{(\mathrm{mean})}
=
\mathbf y^\top R\bar C R\mathbf y
-2\bar{\mathbf b}^\top R\mathbf y
+c,
\end{equation}
and the covariance-level correction is
\begin{equation}
\begin{split}
\mathcal E_{\mathrm{test}}^{(\mathrm{cov})}
&=
\mathbf y^\top
\Big[
R\,\mathbb E_\Theta[\Delta_K R\Delta_K]\,R\bar C R
+R\bar C R\,\mathbb E_\Theta[\Delta_K R\Delta_K]\,R \\
&\qquad\quad
+R\,\mathbb E_\Theta[\Delta_K R\bar C R\Delta_K]\,R
-R\,\mathbb E_\Theta[\Delta_K R\Delta_C]\,R
-R\,\mathbb E_\Theta[\Delta_C R\Delta_K]\,R
\Big]\mathbf y \\
&\qquad
-2\bar{\mathbf b}^\top R\,\mathbb E_\Theta[\Delta_K R\Delta_K]\,R\mathbf y
+2\,\mathbb E_\Theta[\Delta_b^\top R\Delta_K R]\mathbf y.
\end{split}
\end{equation}
The last term contains a resolvent on both sides of \(\Delta_K\); this
factor is retained in the contracted form below.
\end{proposition}

\begin{proof}
Starting from the definition of the test error, we expand the square:
\begin{equation}
\begin{split}
\mathcal E_{\mathrm{test}}(\Theta;D_N)
&=
\int_{\mathcal X}
\Big(
k_\Theta(x)^\top R_\Theta \mathbf y
-y(x)
\Big)^2\,d\mu(x) \\
&=
\int_{\mathcal X}
\Big(
\mathbf y^\top R_\Theta k_\Theta(x)k_\Theta(x)^\top R_\Theta \mathbf y
-2y(x)\,k_\Theta(x)^\top R_\Theta \mathbf y
+y(x)^2
\Big)\,d\mu(x).
\end{split}
\end{equation}
Using the definitions of \(C_\Theta\), \(\mathbf b_\Theta\), and \(c\), this becomes
\begin{equation}
\mathcal E_{\mathrm{test}}(\Theta;D_N)
=
\mathbf y^\top R_\Theta C_\Theta R_\Theta \mathbf y
-2\mathbf b_\Theta^\top R_\Theta \mathbf y
+c.
\end{equation}

Next, write
\begin{equation}
K_\Theta=G+\Delta_K,
\qquad
C_\Theta=\bar C+\Delta_C,
\qquad
\mathbf b_\Theta=\bar{\mathbf b}+\Delta_b,
\end{equation}
so that
\begin{equation}
R_\Theta
=
(K_\Theta+\gamma I_N)^{-1}
=
(R^{-1}+\Delta_K)^{-1}.
\end{equation}
The finite resolvent identity \eqref{eq:finite-resolvent-identity} gives
\begin{equation}
R_\Theta
=
R-R\Delta_KR+R\Delta_KR\Delta_KR+E_{3,\Theta},
\end{equation}
where
\(\|E_{3,\Theta}\|_2\leq
\gamma^{-4}\|\Delta_K\|_2^3\).

Substituting this into
\(
\mathbf y^\top R_\Theta C_\Theta R_\Theta \mathbf y
-2\mathbf b_\Theta^\top R_\Theta \mathbf y
+c
\)
and expanding order by order in the fluctuations, the linear terms vanish after taking expectation because
\(
\mathbb E_\Theta[\Delta_K]
=
\mathbb E_\Theta[\Delta_C]
=
\mathbb E_\Theta[\Delta_b]
=
0
\).
Hence the mean-object contribution is
\begin{equation}
\mathcal E_{\mathrm{test}}^{(\mathrm{mean})}
=
\mathbf y^\top R\bar C R\mathbf y
-2\bar{\mathbf b}^\top R\mathbf y
+c.
\end{equation}

Collecting all quadratic terms in the fluctuations yields
\begin{equation}
\begin{split}
\mathcal E_{\mathrm{test}}^{(\mathrm{cov})}
&=
\mathbf y^\top
\Big[
R\,\mathbb E_\Theta[\Delta_K R\Delta_K]\,R\bar C R
+R\bar C R\,\mathbb E_\Theta[\Delta_K R\Delta_K]\,R \\
&\qquad\quad
+R\,\mathbb E_\Theta[\Delta_K R\bar C R\Delta_K]\,R
-R\,\mathbb E_\Theta[\Delta_K R\Delta_C]\,R
-R\,\mathbb E_\Theta[\Delta_C R\Delta_K]\,R
\Big]\mathbf y \\
&\qquad
-2\bar{\mathbf b}^\top R\,\mathbb E_\Theta[\Delta_K R\Delta_K]\,R\mathbf y
+2\,\mathbb E_\Theta[\Delta_b^\top R\Delta_K R]\mathbf y.
\end{split}
\end{equation}
Every omitted product has total centered degree at least three or contains
\(E_{3,\Theta}\). Since
\(\|R_\Theta\|_2,\|R\|_2\leq\gamma^{-1}\), the stated joint
moment assumptions and Hölder's inequality bound their expected contribution
by \(O(n^{-3/2})\). This proves the claim without an infinite Neumann series.
\end{proof}

The exact means \(\bar C\) and \(\bar{\mathbf b}\) are retained in the
baseline. Their finite-width shifts must therefore not be added a second time
to the covariance correction.

For \(x\in\mathcal X\), define the train--test mean kernel vector
\begin{equation}
g(x):=\mathbb E_\Theta[k_\Theta(x)],
\qquad
G_{x\alpha}:=g_\alpha(x).
\end{equation}
With \(\eta(x):=k_\Theta(x)-g(x)\), define the mixed empirical-kernel
vertices by
\begin{equation}
\begin{split}
V_{(x\alpha)(x'\beta)}
&:=n\,\mathbb E_\Theta[\eta_\alpha(x)\eta_\beta(x')],\\
V_{(\alpha\beta)(x\gamma)}
&:=n\,\mathbb E_\Theta[(\Delta_K)_{\alpha\beta}\eta_\gamma(x)].
\end{split}
\label{eq:mixed-effective-vertices}
\end{equation}
As in \eqref{eq:KV-identification}, these are effective covariance tensors,
not connected preactivation vertices.

\begin{proposition}[Population operators and mixed fluctuation tensors in terms of \(G\) and \(V\)]
\label{cor:Cb-GV}
Let
\begin{equation}
C_\Theta
=
\int_{\mathcal X} k_\Theta(x)k_\Theta(x)^\top\,d\mu(x),
\qquad
\mathbf b_\Theta
=
\int_{\mathcal X} y(x)k_\Theta(x)\,d\mu(x),
\end{equation}
and write
\begin{equation}
C_\Theta=\bar C+\Delta_C,
\qquad
\mathbf b_\Theta=\bar{\mathbf b}+\Delta_b,
\end{equation}
where
\begin{equation}
\bar C:=\mathbb E_\Theta[C_\Theta],
\qquad
\bar{\mathbf b}:=\mathbb E_\Theta[\mathbf b_\Theta].
\end{equation}
By definition, the mixed train--test kernel fluctuations satisfy
\begin{equation}
\mathbb E_\Theta[\eta_\alpha(x)\eta_\beta(x')]
=
\frac{1}{n}V_{(x\alpha)(x'\beta)},
\label{eq:mixed-eta-eta}
\end{equation}
and
\begin{equation}
\mathbb E_\Theta[(\Delta_K)_{\alpha\beta}\eta_\gamma(x)]
=
\frac{1}{n}V_{(\alpha\beta)(x\gamma)}.
\label{eq:mixed-delta-eta}
\end{equation}
Under the integrable-envelope condition needed to interchange the population
and ensemble expectations, the following mean identities are exact:
\begin{equation}
\bar C_{\alpha\beta}
=
\int_{\mathcal X}
G_{x\alpha}G_{x\beta}\,d\mu(x)
+
\frac{1}{n}
\int_{\mathcal X}V_{(x\alpha)(x\beta)}\,d\mu(x),
\label{eq:Cbar-GV}
\end{equation}
and
\begin{equation}
\bar b_\alpha
=
\int_{\mathcal X}y(x)\,G_{x\alpha}\,d\mu(x).
\label{eq:bbar-G}
\end{equation}
Assume in addition that the mixed centered third moments involving one
\(\Delta_K\) and two \(\eta\)'s are \(O(n^{-3/2})\), uniformly under the
same envelope. Then, for every \(A\in\mathbb R^{N\times N}\),
\begin{equation}
\mathbb E_\Theta[\Delta_K A\Delta_C]_{\alpha\delta}
=
\frac{1}{n}
\sum_{\beta,\gamma=1}^N
A_{\beta\gamma}
\int_{\mathcal X}
\Big(
V_{(\alpha\beta)(x\gamma)}G_{x\delta}
+
V_{(\alpha\beta)(x\delta)}G_{x\gamma}
\Big)\,d\mu(x)
+
O(n^{-3/2}),
\label{eq:delta-A-deltaC}
\end{equation}
\begin{equation}
\mathbb E_\Theta[\Delta_C A\Delta_K]_{\alpha\delta}
=
\frac{1}{n}
\sum_{\beta,\gamma=1}^N
A_{\beta\gamma}
\int_{\mathcal X}
\Big(
G_{x\alpha}V_{(x\beta)(\gamma\delta)}
+
V_{(x\alpha)(\gamma\delta)}G_{x\beta}
\Big)\,d\mu(x)
+
O(n^{-3/2}),
\label{eq:deltaC-A-delta}
\end{equation}
and
\begin{equation}
\mathbb E_\Theta[\Delta_b^\top A\Delta_K]_\delta
=
\frac{1}{n}
\sum_{\alpha,\beta=1}^N
A_{\alpha\beta}
\int_{\mathcal X}
y(x)\,
V_{(x\alpha)(\beta\delta)}\,d\mu(x).
\label{eq:deltab-A-delta}
\end{equation}
\end{proposition}

\begin{proof}
For \(x\in\mathcal X\), define the train--test mean kernel vector by
\begin{equation}
g(x):=\mathbb E_\Theta[k_\Theta(x)],
\qquad
g_\alpha(x)=G_{x\alpha}.
\end{equation}
Then
\begin{equation}
\eta_\alpha(x):=k_{\Theta,\alpha}(x)-G_{x\alpha}
\end{equation}
satisfies
\begin{equation}
\mathbb E_\Theta[\eta_\alpha(x)]=0
\end{equation}
by definition. Hence
\begin{equation}
C_{\Theta,\alpha\beta}
=
\int_{\mathcal X}
\bigl(g_\alpha(x)+\eta_\alpha(x)\bigr)
\bigl(g_\beta(x)+\eta_\beta(x)\bigr)
\,d\mu(x).
\end{equation}
Taking the ensemble expectation gives
\begin{equation}
\bar C_{\alpha\beta}
=
\mathbb E_\Theta[C_{\Theta,\alpha\beta}]
=
\int_{\mathcal X}
g_\alpha(x)g_\beta(x)\,d\mu(x)
+
\int_{\mathcal X}
\mathbb E_\Theta[\eta_\alpha(x)\eta_\beta(x)]
\,d\mu(x).
\end{equation}
Using
\begin{equation}
\mathbb E_\Theta[\eta_\alpha(x)\eta_\beta(x)]
=
\frac{1}{n}V_{(x\alpha)(x\beta)},
\end{equation}
we obtain
\begin{equation}
\bar C_{\alpha\beta}
=
\int_{\mathcal X}
G_{x\alpha}G_{x\beta}\,d\mu(x)
+
\frac{1}{n}
\int_{\mathcal X}
V_{(x\alpha)(x\beta)}
\,d\mu(x).
\end{equation}
Similarly,
\begin{equation}
b_{\Theta,\alpha}
=
\int_{\mathcal X}
y(x)\bigl(g_\alpha(x)+\eta_\alpha(x)\bigr)\,d\mu(x).
\end{equation}
Therefore,
\begin{equation}
\bar b_\alpha
=
\mathbb E_\Theta[b_{\Theta,\alpha}]
=
\int_{\mathcal X}
y(x)G_{x\alpha}\,d\mu(x)
+
\int_{\mathcal X}
y(x)\mathbb E_\Theta[\eta_\alpha(x)]\,d\mu(x).
\end{equation}
Since \(\mathbb E_\Theta[\eta_\alpha(x)]=0\) exactly, we get
\begin{equation}
\bar b_\alpha
=
\int_{\mathcal X}y(x)G_{x\alpha}\,d\mu(x).
\end{equation}

It remains to compute the mixed contractions. Since
\begin{equation}
\Delta_C=C_\Theta-\bar C,
\end{equation}
we have
\begin{equation}
\Delta_{C,\gamma\delta}
=
\int_{\mathcal X}
\Bigl(
G_{x\gamma}\eta_\delta(x)
+
\eta_\gamma(x)G_{x\delta}
+
\eta_\gamma(x)\eta_\delta(x)
-
\mathbb E_\Theta[\eta_\gamma(x)\eta_\delta(x)]
\Bigr)
\,d\mu(x).
\end{equation}
The quadratic centered term
\begin{equation}
\eta_\gamma(x)\eta_\delta(x)
-
\mathbb E_\Theta[\eta_\gamma(x)\eta_\delta(x)]
\end{equation}
has an \(O(n^{-3/2})\) expected contraction with one additional
\(\Delta_K\) by assumption. Therefore only the two terms linear in
\(\eta\) contribute to the contraction at order \(n^{-1}\).

Now
\begin{equation}
(\Delta_K A\Delta_C)_{\alpha\delta}
=
\sum_{\beta,\gamma=1}^N
(\Delta_K)_{\alpha\beta}A_{\beta\gamma}\Delta_{C,\gamma\delta}.
\end{equation}
Substituting the leading expression for \(\Delta_C\), we obtain
\begin{equation}
\begin{split}
\mathbb E_\Theta[(\Delta_K A\Delta_C)_{\alpha\delta}]
={}&
\sum_{\beta,\gamma=1}^N
A_{\beta\gamma}
\int_{\mathcal X}
\Bigl(
G_{x\gamma}\mathbb E_\Theta[(\Delta_K)_{\alpha\beta}\eta_\delta(x)]\\
&\qquad\qquad
+
G_{x\delta}\mathbb E_\Theta[(\Delta_K)_{\alpha\beta}\eta_\gamma(x)]
\Bigr)
\,d\mu(x)
+
O(n^{-3/2}).
\end{split}
\end{equation}
Using
\begin{equation}
\mathbb E_\Theta[(\Delta_K)_{\alpha\beta}\eta_\gamma(x)]
=
\frac{1}{n}V_{(\alpha\beta)(x\gamma)},
\end{equation}
we get
\begin{equation}
\mathbb E_\Theta[\Delta_K A\Delta_C]_{\alpha\delta}
=
\frac{1}{n}
\sum_{\beta,\gamma=1}^N
A_{\beta\gamma}
\int_{\mathcal X}
\Big(
V_{(\alpha\beta)(x\gamma)}G_{x\delta}
+
V_{(\alpha\beta)(x\delta)}G_{x\gamma}
\Big)
\,d\mu(x)
+
O(n^{-3/2}).
\end{equation}

The proof of the second contraction is analogous. Since
\begin{equation}
(\Delta_C A\Delta_K)_{\alpha\delta}
=
\sum_{\beta,\gamma=1}^N
\Delta_{C,\alpha\beta}A_{\beta\gamma}(\Delta_K)_{\gamma\delta},
\end{equation}
and retaining its terms linear in \(\eta\), we obtain
\begin{equation}
\begin{split}
\mathbb E_\Theta[(\Delta_C A\Delta_K)_{\alpha\delta}]
={}&
\sum_{\beta,\gamma=1}^N
A_{\beta\gamma}
\int_{\mathcal X}
\Bigl(
G_{x\alpha}\mathbb E_\Theta[\eta_\beta(x)(\Delta_K)_{\gamma\delta}]\\
&\qquad\qquad
+
G_{x\beta}\mathbb E_\Theta[\eta_\alpha(x)(\Delta_K)_{\gamma\delta}]
\Bigr)
\,d\mu(x)
+
O(n^{-3/2}).
\end{split}
\end{equation}
Therefore,
\begin{equation}
\mathbb E_\Theta[\Delta_C A\Delta_K]_{\alpha\delta}
=
\frac{1}{n}
\sum_{\beta,\gamma=1}^N
A_{\beta\gamma}
\int_{\mathcal X}
\Big(
G_{x\alpha}V_{(x\beta)(\gamma\delta)}
+
V_{(x\alpha)(\gamma\delta)}G_{x\beta}
\Big)
\,d\mu(x)
+
O(n^{-3/2}).
\end{equation}

Finally, since
\begin{equation}
\Delta_b=\mathbf b_\Theta-\bar{\mathbf b},
\end{equation}
we have
\begin{equation}
\Delta_{b,\alpha}
=
\int_{\mathcal X}y(x)\eta_\alpha(x)\,d\mu(x).
\end{equation}
Thus
\begin{equation}
(\Delta_b^\top A\Delta_K)_\delta
=
\sum_{\alpha,\beta=1}^N
\Delta_{b,\alpha}A_{\alpha\beta}(\Delta_K)_{\beta\delta}.
\end{equation}
Taking the expectation gives
\begin{equation}
\mathbb E_\Theta[(\Delta_b^\top A\Delta_K)_\delta]
=
\sum_{\alpha,\beta=1}^N
A_{\alpha\beta}
\int_{\mathcal X}
y(x)\mathbb E_\Theta[\eta_\alpha(x)(\Delta_K)_{\beta\delta}]
\,d\mu(x).
\end{equation}
Using the mixed covariance relation,
\begin{equation}
\mathbb E_\Theta[\eta_\alpha(x)(\Delta_K)_{\beta\delta}]
=
\frac{1}{n}V_{(x\alpha)(\beta\delta)},
\end{equation}
we obtain
\begin{equation}
\mathbb E_\Theta[\Delta_b^\top A\Delta_K]_\delta
=
\frac{1}{n}
\sum_{\alpha,\beta=1}^N
A_{\alpha\beta}
\int_{\mathcal X}
y(x)
V_{(x\alpha)(\beta\delta)}
\,d\mu(x).
\end{equation}
This completes the proof.
\end{proof}

\begin{corollary}[Test covariance correction in terms of \(G\) and \(V\)]
\label{cor:test-error-GV}
Under the assumptions of Proposition~\ref{cor:Cb-GV}, the mean-object
baseline is
\begin{equation}
\mathcal E_{\mathrm{test}}^{(\mathrm{mean})}
=
\mathbf y^\top R\bar C R\mathbf y
-
2\bar{\mathbf b}^\top R\mathbf y
+
c,
\qquad
R=(G+\gamma I_N)^{-1},
\end{equation}
with \(\bar C\) and \(\bar{\mathbf b}\) given by
\eqref{eq:Cbar-GV} and \eqref{eq:bbar-G}. Its covariance-level correction is
\begin{equation}
\begin{split}
\mathcal E_{\mathrm{test}}^{(\mathrm{cov})}
=
\frac{1}{n}\Big[&
\mathbf y^\top R(V\star R)R\bar C R\mathbf y
+
\mathbf y^\top R\bar C R(V\star R)R\mathbf y \\
&+
\mathbf y^\top R(V\star(R\bar C R))R\mathbf y
-
\mathbf y^\top R\,\mathcal M_{KC}[R]\,R\mathbf y
-
\mathbf y^\top R\,\mathcal M_{CK}[R]\,R\mathbf y \\
&-
2\bar{\mathbf b}^\top R(V\star R)R\mathbf y
+
2\,\mathcal M_{bK}[R]R\mathbf y
\Big]
+
O(n^{-3/2}).
\end{split}
\label{eq:Ttest2-GV}
\end{equation}
where
\begin{equation}
\bigl(\mathcal M_{KC}[A]\bigr)_{\alpha\delta}
:=
\sum_{\beta,\gamma=1}^N
A_{\beta\gamma}
\int_{\mathcal X}
\Big(
V_{(\alpha\beta)(x\gamma)}G_{x\delta}
+
V_{(\alpha\beta)(x\delta)}G_{x\gamma}
\Big)\,d\mu(x),
\end{equation}
\begin{equation}
\bigl(\mathcal M_{CK}[A]\bigr)_{\alpha\delta}
:=
\sum_{\beta,\gamma=1}^N
A_{\beta\gamma}
\int_{\mathcal X}
\Big(
G_{x\alpha}V_{(x\beta)(\gamma\delta)}
+
V_{(x\alpha)(\gamma\delta)}G_{x\beta}
\Big)\,d\mu(x),
\end{equation}
and
\begin{equation}
\bigl(\mathcal M_{bK}[A]\bigr)_\delta
:=
\sum_{\alpha,\beta=1}^N
A_{\alpha\beta}
\int_{\mathcal X}
y(x)\,V_{(x\alpha)(\beta\delta)}\,d\mu(x).
\end{equation}
Thus, for fixed \(D_N,N,L,\gamma\), the centered correction is
\(O(n^{-1})\) and is determined by the mean objects together with the
train--train and mixed empirical-kernel covariance vertices. The final
\(R\) in the \(\mathcal M_{bK}\) term is essential.
\end{corollary}

\subsection{Expansion of the generalization gap}
We now turn to the generalization gap, defined as the difference between the test error
and the training error:
\begin{equation}
\mathcal E_{\mathrm{gen}}(\Theta;D_N)
:=
\mathcal E_{\mathrm{test}}(\Theta;D_N)
-
\mathcal E_{\mathrm{tr}}(\Theta;D_N).
\end{equation}
Using the kernel representations derived above, we have
\begin{equation}
\mathcal E_{\mathrm{tr}}(\Theta;D_N)
=
\frac{\gamma^{2}}{N}\,\mathbf y^\top (K_\Theta+\gamma I_N)^{-2}\mathbf y,
\end{equation}
and
\begin{equation}
\mathcal E_{\mathrm{test}}(\Theta;D_N)
=
\mathbf y^\top R_\Theta C_\Theta R_\Theta \mathbf y
-2\mathbf b_\Theta^\top R_\Theta \mathbf y
+c,
\end{equation}
where
\begin{equation}
R_\Theta:=(K_\Theta+\gamma I_N)^{-1}.
\end{equation}
\begin{equation}
C_\Theta:=\int_{\mathcal X} k_\Theta(x)k_\Theta(x)^\top\,d\mu(x),
\qquad
\mathbf b_\Theta:=\int_{\mathcal X} y(x)\,k_\Theta(x)\,d\mu(x),
\qquad
c:=\int_{\mathcal X} y(x)^2\,d\mu(x).
\end{equation}
Therefore,
\begin{equation}
\mathcal E_{\mathrm{gen}}(\Theta;D_N)
=
\mathbf y^\top R_\Theta C_\Theta R_\Theta \mathbf y
-2\mathbf b_\Theta^\top R_\Theta \mathbf y
+c
-
\frac{\gamma^2}{N}\,\mathbf y^\top R_\Theta^2 \mathbf y.
\end{equation}

\begin{proposition}[Centered expansion of the conditional generalization gap]
\label{prop:loop-expansion-generalization-gap}
Let
\begin{equation}
K_\Theta=G+\Delta_K,
\qquad
C_\Theta=\bar C+\Delta_C,
\qquad
\mathbf b_\Theta=\bar{\mathbf b}+\Delta_b,
\end{equation}
where
\begin{equation}
G:=\mathbb E_\Theta[K_\Theta],
\qquad
\bar C:=\mathbb E_\Theta[C_\Theta],
\qquad
\bar{\mathbf b}:=\mathbb E_\Theta[\mathbf b_\Theta],
\end{equation}
and assume
\begin{equation}
\mathbb E_\Theta[\Delta_K]=0,
\qquad
\mathbb E_\Theta[\Delta_C]=0,
\qquad
\mathbb E_\Theta[\Delta_b]=0.
\end{equation}
Define further
\begin{equation}
\gamma:=\frac{N\lambda}{n},
\qquad
R:=(G+\gamma I_N)^{-1}.
\end{equation}
Then the expected generalization gap admits the expansion
\begin{equation}
\mathbb E_\Theta[\mathcal E_{\mathrm{gen}}]
=
\mathcal E_{\mathrm{gen}}^{(\mathrm{mean})}
+
\mathcal E_{\mathrm{gen}}^{(\mathrm{cov})}
+
O(n^{-3/2}),
\end{equation}
where
\begin{equation}
\mathcal E_{\mathrm{gen}}^{(\mathrm{mean})}
:=
\mathcal E_{\mathrm{test}}^{(\mathrm{mean})}
-\mathcal E_{\mathrm{tr}}^{(\mathrm{mean})},
\end{equation}
and
\begin{equation}
\mathcal E_{\mathrm{gen}}^{(\mathrm{cov})}
:=
\mathcal E_{\mathrm{test}}^{(\mathrm{cov})}
-\mathcal E_{\mathrm{tr}}^{(\mathrm{cov})}.
\end{equation}
The training contribution depends only on the train--train covariance vertex.
The test contribution also contains \(\mathcal M_{KC}\),
\(\mathcal M_{CK}\), and \(\mathcal M_{bK}\), and therefore requires
off-training fluctuation information.
\end{proposition}

\begin{corollary}[Conditional gap in terms of \(G\) and \(V\)]
\label{cor:gap-GV}
Under the assumptions of Corollaries~\ref{cor:training-error-GV} and \ref{cor:test-error-GV}, the expected generalization gap admits the expansion
\begin{equation}
\mathbb E_\Theta[\mathcal E_{\mathrm{gen}}]
=
\mathcal E_{\mathrm{gen}}^{(\mathrm{mean})}
+
\mathcal E_{\mathrm{gen}}^{(\mathrm{cov})}
+
O\Big(n^{-3/2}\Big),
\end{equation}
where
\begin{equation}
\mathcal E_{\mathrm{gen}}^{(\mathrm{mean})}
=
\mathbf y^\top R\bar C R\mathbf y
-
2\bar{\mathbf b}^\top R\mathbf y
+
c
-
\frac{\gamma^2}{N}\mathbf y^\top R^2\mathbf y,
\end{equation}
and
\begin{equation}
\mathcal E_{\mathrm{gen}}^{(\mathrm{cov})}
=
\mathcal E_{\mathrm{test}}^{(\mathrm{cov})}
-\mathcal E_{\mathrm{tr}}^{(\mathrm{cov})}.
\label{eq:Tgen2-GV}
\end{equation}
The symbol \(V\) here includes both train--train and mixed empirical-kernel
covariance tensors. Train-restricted \(G\) and \(V\) alone do not determine
the test or gap correction.
\end{corollary}

\section{Spectral Form and Fixed-Design Bounds}

We now express the mean-object and covariance terms in the eigenbasis of
\(G\). This exposes the resolvent denominators and the empirical-kernel
covariance contractions. The formulas describe fixed-\(N\), fixed-depth
scaling in \(n\); they do not by themselves imply a joint \(N,n\) law.
Let
\begin{equation}
G = U D U^\top,
\qquad
D=\mathrm{diag}(\rho_1,\dots,\rho_N),
\qquad
R:=(G+\gamma I_N)^{-1}=U \Lambda U^\top,
\end{equation}
where
\begin{equation}
\Lambda=\mathrm{diag}\!\left(\frac{1}{\rho_1+\gamma},\dots,\frac{1}{\rho_N+\gamma}\right).
\end{equation}
For any vector \(a\in\mathbb R^N\) and matrix \(A\in\mathbb R^{N\times N}\), define
\begin{equation}
\tilde a := U^\top a,
\qquad
\tilde A := U^\top AU.
\end{equation}
We also define the vertex in the eigenbasis of \(G\) by
\begin{equation}
\tilde V_{(ij)(k\ell)}
:=
\sum_{\alpha,\beta,\gamma,\delta=1}^N
U_{\alpha i}U_{\beta j}U_{\gamma k}U_{\delta \ell}
V_{(\alpha\beta)(\gamma\delta)}.
\end{equation}

\begin{proposition}[Spectral form of the training error]
\label{prop:spectral-training}
Let \(G=UD U^\top \) with \(D_{ii}=\rho_i\), and let \(\tilde y=U^\top y\).
For compactness in this section, write
\(\mathcal E_{\mathrm{tr}}^{(0)}
:=\mathcal E_{\mathrm{tr}}^{(\mathrm{mean})}\) and
\(\mathcal E_{\mathrm{tr}}^{(1)}
:=\mathcal E_{\mathrm{tr}}^{(\mathrm{cov})}\).
Then
\begin{equation}
\mathbb E_\Theta[\mathcal E_{\mathrm{tr}}]
=
\mathcal E_{\mathrm{tr}}^{(0)}
+
\mathcal E_{\mathrm{tr}}^{(1)}
+
O\Big(n^{-3/2}\Big),
\end{equation}
where
\begin{equation}
\mathcal E_{\mathrm{tr}}^{(0)}
=
\frac{\gamma^2}{N}\sum_{i=1}^N
\frac{\tilde y_i^2}{(\rho_i+\gamma)^2},
\end{equation}
and
\begin{equation}
\begin{split}
  &  \mathcal E_{\mathrm{tr}}^{(1)}
=
\frac{\gamma^2}{Nn}
\sum_{i,j,k=1}^N
\tilde y_i\tilde y_j\,
\tilde V_{(ik)(kj)}
\\&\left[
\frac{1}{(\rho_i+\gamma)^2(\rho_j+\gamma)(\rho_k+\gamma)}
+
\frac{1}{(\rho_i+\gamma)(\rho_j+\gamma)^2(\rho_k+\gamma)}
+
\frac{1}{(\rho_i+\gamma)(\rho_j+\gamma)(\rho_k+\gamma)^2}
\right].
\end{split}
\end{equation}
\end{proposition}
\begin{proof}
By Corollary~\ref{cor:training-error-GV}, the expected training error up to the leading
finite-width correction is
\begin{equation}
\mathbb E_\Theta[\mathcal E_{\mathrm{tr}}]
=
\frac{\gamma^2}{N}
\mathbf y^\top
\left[
R^2
+
\frac1n
\left\{
R^2(V\star R)R
+
R(V\star R^2)R
+
R(V\star R)R^2
\right\}
\right]\mathbf y
+
O(n^{-3/2}),
\end{equation}
where
\begin{equation}
R:=(G+\gamma I_N)^{-1}.
\end{equation}
Let
\begin{equation}
G=UDU^\top,
\qquad
D=\mathrm{diag}(\rho_1,\ldots,\rho_N).
\end{equation}
Then
\begin{equation}
R=U\Lambda U^\top,
\qquad
\Lambda=\mathrm{diag}(r_1,\ldots,r_N),
\qquad
r_i:=\frac{1}{\rho_i+\gamma}.
\end{equation}
Writing \(\tilde y:=U^\top y\), the mean-object term is
\begin{equation}
\mathcal E_{\mathrm{tr}}^{(0)}
=
\frac{\gamma^2}{N}
\mathbf y^\top R^2\mathbf y
=
\frac{\gamma^2}{N}
\tilde{\mathbf y}^{\top}\Lambda^2\tilde{\mathbf y}
=
\frac{\gamma^2}{N}
\sum_{i=1}^N
\frac{\tilde y_i^2}{(\rho_i+\gamma)^2}.
\end{equation}

It remains to compute the spectral form of the covariance term. Recall that
\begin{equation}
(V\star A)_{\alpha\delta}
=
\sum_{\beta,\gamma=1}^N
V_{(\alpha\beta)(\gamma\delta)}A_{\beta\gamma}.
\end{equation}
In the eigenbasis of \(G\), define
\begin{equation}
\widetilde{(V\star A)}
:=
U^\top (V\star A)U.
\end{equation}
Then, using \(A=R\), we obtain
\begin{equation}
\begin{split}
\widetilde{(V\star R)}_{ij}
&=
\sum_{\alpha,\delta=1}^N
U_{\alpha i}(V\star R)_{\alpha\delta}U_{\delta j} \\
&=
\sum_{\alpha,\beta,\gamma,\delta=1}^N
U_{\alpha i}
V_{(\alpha\beta)(\gamma\delta)}
R_{\beta\gamma}
U_{\delta j}.
\end{split}
\end{equation}
Since
\begin{equation}
R_{\beta\gamma}
=
\sum_{k=1}^N
U_{\beta k}r_k U_{\gamma k},
\end{equation}
we get
\begin{equation}
\widetilde{(V\star R)}_{ij}
=
\sum_{k=1}^N
r_k
\sum_{\alpha,\beta,\gamma,\delta=1}^N
U_{\alpha i}U_{\beta k}U_{\gamma k}U_{\delta j}
V_{(\alpha\beta)(\gamma\delta)}.
\end{equation}
By the definition of the vertex in the eigenbasis, this is
\begin{equation}
\widetilde{(V\star R)}_{ij}
=
\sum_{k=1}^N
r_k \tilde V_{(ik)(kj)}.
\end{equation}
Similarly,
\begin{equation}
\widetilde{(V\star R^2)}_{ij}
=
\sum_{k=1}^N
r_k^2 \tilde V_{(ik)(kj)}.
\end{equation}

Now consider the three one-loop contributions separately. First,
\begin{equation}
\begin{split}
\mathbf y^\top R^2(V\star R)R\mathbf y
&=
\tilde{\mathbf y}^{\top}
\Lambda^2
\widetilde{(V\star R)}
\Lambda
\tilde{\mathbf y} \\
&=
\sum_{i,j=1}^N
\tilde y_i\tilde y_j
r_i^2r_j
\widetilde{(V\star R)}_{ij} \\
&=
\sum_{i,j,k=1}^N
\tilde y_i\tilde y_j
r_i^2r_jr_k
\tilde V_{(ik)(kj)}.
\end{split}
\end{equation}
Second,
\begin{equation}
\begin{split}
\mathbf y^\top R(V\star R^2)R\mathbf y
&=
\tilde{\mathbf y}^{\top}
\Lambda
\widetilde{(V\star R^2)}
\Lambda
\tilde{\mathbf y} \\
&=
\sum_{i,j,k=1}^N
\tilde y_i\tilde y_j
r_ir_jr_k^2
\tilde V_{(ik)(kj)}.
\end{split}
\end{equation}
Third,
\begin{equation}
\begin{split}
\mathbf y^\top R(V\star R)R^2\mathbf y
&=
\tilde{\mathbf y}^{\top}
\Lambda
\widetilde{(V\star R)}
\Lambda^2
\tilde{\mathbf y} \\
&=
\sum_{i,j,k=1}^N
\tilde y_i\tilde y_j
r_ir_j^2r_k
\tilde V_{(ik)(kj)}.
\end{split}
\end{equation}
Combining the three terms and substituting
\begin{equation}
r_i=\frac{1}{\rho_i+\gamma}
\end{equation}
gives
\begin{equation}
\begin{split}
\mathcal E_{\mathrm{tr}}^{(1)}
=
\frac{\gamma^2}{Nn}
\sum_{i,j,k=1}^N
\tilde y_i\tilde y_j
\tilde V_{(ik)(kj)}
\Bigg[
&
\frac{1}{(\rho_i+\gamma)^2(\rho_j+\gamma)(\rho_k+\gamma)}
\\
&+
\frac{1}{(\rho_i+\gamma)(\rho_j+\gamma)^2(\rho_k+\gamma)}
\\
&+
\frac{1}{(\rho_i+\gamma)(\rho_j+\gamma)(\rho_k+\gamma)^2}
\Bigg].
\end{split}
\end{equation}
This proves the claimed spectral representation.
\end{proof}
\begin{corollary}[Resolvent scaling of the training error]
\label{cor:training-resolvent-scaling}
Let the contraction-map norm be
\begin{equation}
\kappa_V
:=
\sup_{A\ne0}
\frac{\|V\star A\|_2}{\|A\|_2}.
\end{equation}
Since \(G\) is positive semidefinite and \(\gamma>0\), the denominators are
bounded below by \(\rho_{\min}+\gamma\), where
\begin{equation}
\rho_{\min}:=\min_i \rho_i .
\end{equation}
Therefore,
\begin{equation}
\mathcal E_{\mathrm{tr}}^{(0)}
\le
\frac{\gamma^2}{N(\rho_{\min}+\gamma)^2}
\|y\|^2.
\end{equation}
The one-loop correction satisfies
\begin{equation}
\left|\mathcal E_{\mathrm{tr}}^{(1)}\right|
\le
\frac{3\gamma^2\kappa_V\|y\|^2}
{Nn(\rho_{\min}+\gamma)^4}.
\end{equation}
The dependence on \(N\) is therefore explicit except for any dependence
carried by \(\kappa_V\); no dimension-dependent factor is hidden in the
big-\(O\) notation.
\end{corollary}
The bound follows directly from
\(\|V\star A\|_2\leq\kappa_V\|A\|_2\) and contains four resolvent
factors. It is informative only when \(\kappa_V\) is itself controlled.

To express the test error in the same eigenbasis, project the exact mean
population objects onto the eigenvectors of \(G\):
\begin{equation}
\tilde g_i(x):=\sum_{\alpha=1}^N U_{\alpha i}G_{x\alpha},
\qquad
\tilde b_i
:=
\int_{\mathcal X} y(x)\tilde g_i(x)\,d\mu(x),
\end{equation}
and
\begin{equation}
\tilde C_{ij}
:=(U^\top\bar C U)_{ij}
=
\int_{\mathcal X}\tilde g_i(x)\tilde g_j(x)\,d\mu(x)
+\frac1n\int_{\mathcal X}\tilde V_{(xi)(xj)}\,d\mu(x).
\end{equation}
For mixed vertices, we project only the training indices onto the eigenbasis of
\(G\). For example,
\begin{equation}
\tilde V_{(xi)(x'j)}
:=
\sum_{\alpha,\beta=1}^N
U_{\alpha i}U_{\beta j}
V_{(x\alpha)(x'\beta)},
\end{equation}
and
\begin{equation}
\tilde V_{(xi)(jk)}
:=
\sum_{\alpha,\beta,\gamma=1}^N
U_{\alpha i}U_{\beta j}U_{\gamma k}
V_{(x\alpha)(\beta\gamma)}.
\end{equation}
\begin{proposition}[Spectral structure of the test error]
\label{prop:spectral-test}
Let \(R=(G+\gamma I_N)^{-1}\). Then
\begin{equation}
\mathbb E_\Theta[\mathcal E_{\mathrm{test}}]
=
\mathcal E_{\mathrm{test}}^{(\mathrm{mean})}
+
\mathcal E_{\mathrm{test}}^{(\mathrm{cov})}
+
O\Big(n^{-3/2}\Big).
\end{equation}
where
\begin{equation}
\mathcal E_{\mathrm{test}}^{(\mathrm{mean})}
=
\sum_{i,j=1}^N
\frac{\tilde y_i\tilde C_{ij}\tilde y_j}
{(\rho_i+\gamma)(\rho_j+\gamma)}
-
2\sum_{i=1}^N
\frac{\tilde b_i\tilde y_i}{\rho_i+\gamma}
+
c.
\end{equation}
The covariance term is \(O(n^{-1})\). Its components consist of
sums whose resolvent factors take the form
\begin{equation}
\frac{1}{
(\rho_{i_1}+\gamma)\cdots(\rho_{i_m}+\gamma)
},
\qquad
m=2,3,4,
\end{equation}
with coefficients determined by contractions of \(\tilde C\), \(\tilde b\),
and the mixed empirical-kernel covariance vertices. Hence the spectral
amplification is again governed by small denominators \((\rho_i+\gamma)^{-1}\).
Unlike the training correction, however, the coefficient structure is not
determined solely by the train--train vertex, but also by population and
mixed train--test fluctuation geometry.
\end{proposition}
\begin{proof}
The mean-object formula follows from
\(R=U\Lambda U^\top\), \(\tilde C=U^\top\bar C U\), and
\(\tilde b=U^\top\bar b\). For the covariance terms, the basic
train--train contraction becomes
\begin{equation}
\widetilde{(V\star R)}_{ij}
=
\sum_{k=1}^N
\frac{\tilde V_{(ik)(kj)}}{\rho_k+\gamma},
\end{equation}
while
\begin{equation}
\widetilde{(V\star(R\bar C R))}_{ij}
=
\sum_{k,\ell=1}^N
\frac{\tilde C_{k\ell}\tilde V_{(ik)(\ell j)}}
{(\rho_k+\gamma)(\rho_\ell+\gamma)}.
\end{equation}
Thus the three pure train--train terms contain four resolvent factors.
The maps \(\mathcal M_{KC}[R]\) and \(\mathcal M_{CK}[R]\) are linear
in \(R\) and are multiplied by a resolvent on each side, so their terms
contain three factors. The term
\(\mathcal M_{bK}[R]R\mathbf y\) contains two; the trailing \(R\) is
required by \(\mathbb E[\Delta_b^\top R\Delta_KR]\mathbf y\).
Orthogonal changes of basis do not alter the order of the moment remainder.
This proves the stated spectral structure.
\end{proof}
\begin{corollary}[Spectral decomposition of the conditional gap]
\label{thm:spectral-gap}
The expected conditional generalization gap satisfies
\begin{equation}
\mathbb E_\Theta[\mathcal E_{\mathrm{gen}}]
=
\mathcal E_{\mathrm{gen}}^{(\mathrm{mean})}
+
\mathcal E_{\mathrm{gen}}^{(\mathrm{cov})}
+
O(n^{-3/2}),
\end{equation}
where
\begin{equation}
\begin{split}
\mathcal E_{\mathrm{gen}}^{(\mathrm{mean})}
={}&
\sum_{i,j=1}^N
\frac{\tilde y_i\tilde C_{ij}\tilde y_j}
{(\rho_i+\gamma)(\rho_j+\gamma)}
-2\sum_{i=1}^N
\frac{\tilde b_i\tilde y_i}{\rho_i+\gamma}
+c\\
&-\frac{\gamma^2}{N}\sum_{i=1}^N
\frac{\tilde y_i^2}{(\rho_i+\gamma)^2},
\end{split}
\end{equation}
and
\begin{equation}
\mathcal E_{\mathrm{gen}}^{(\mathrm{cov})}
=
\mathcal E_{\mathrm{test}}^{(\mathrm{cov})}
-\mathcal E_{\mathrm{tr}}^{(\mathrm{cov})}.
\end{equation}
This identity is a decomposition, not a comparison theorem: without an
additional norm or ordering assumption it does not imply that the gap is more
sensitive than either constituent error.
\end{corollary}
\begin{proposition}[Non-identifiability from train-restricted data]
\label{prop:nonidentifiability}
There is no universal functional of
\begin{equation}
\left(G|_{D_N^2},\,V|_{D_N^4},\,\mathbf y,\,\gamma\right)
\end{equation}
that recovers \(\mathcal E_{\mathrm{test}}^{(\mathrm{cov})}\) for every
admissible extension of the random feature map from \(D_N\) to
\(\mathcal X\). Consequently, the conditional gap correction is not
identifiable from train-restricted spectral data alone.
\end{proposition}
\begin{proof}
Consider an input space containing the training points and an additional
population point \(x_\star\). Keep the random feature vectors at every
training point identical in two models. In the first model set
\(\phi(x_\star;\Theta)=0\); in the second set
\(\phi(x_\star;\Theta)=\phi(x_1;\Theta)\). The two models have the same
train-restricted \(G\), \(V\), labels, and training correction, but their
\(k_\Theta(x_\star)\), and hence \(C_\Theta\), \(b_\Theta\), and the mixed
covariance maps, differ. Choosing a population measure with positive mass at
\(x_\star\) yields different test corrections for generic targets. Therefore
no functional of the train-restricted quantities can recover the test
correction for all admissible extensions.
\end{proof}

\section{Numerical Evaluation and Mechanisms}
\label{sec:numerics}

The numerical study is designed to evaluate three aspects that are specific to
the observable-level expansion. First, we assess whether the covariance term
removes the leading finite-width discrepancy rather than merely exhibiting the
expected explicit factor of $1/n$. Second, we examine the numerical necessity
of the mixed train--test tensors appearing in Corollary~\ref{cor:test-error-GV}. Third, we
investigate whether the resolvent-amplified fluctuation provides a reliable
diagnostic for the regime in which the second-order truncation remains
accurate. These aspects are evaluated in a controlled regularization regime,
with the weakly regularized setting retained as a stress test.

\subsection{Setup, paired estimator, and uncertainty}
\label{sec:numerical-setup}

We consider one-dimensional inputs. A training design with
\(N_{\mathrm{train}}=32\) and a test grid with \(N_{\mathrm{test}}=1024\)
are drawn from the standard Gaussian distribution and held fixed across all
feature realizations in a setting. The targets are
\begin{equation}
y(x)=\sin(2x),\qquad
y(x)=0.4x^3-0.6x^2+0.2x,\qquad
y(x)=|x|.
\end{equation}
To keep the label scale fixed when \(N\) changes, the three targets are
standardized by reference means and standard deviations estimated once from
\(2^{16}\) standard-Gaussian inputs, rather than by the training-sample
statistics.

The frozen feature maps are fully connected networks with the Gaussian
initialization in \eqref{ensemble}. The main setting uses \(C_W=1\),
\(C_b=2.5\times10^{-3}\), and either \(\tanh\) or ReLU activations at every
feature layer, including the last layer. Here \(C_b\) is a variance; the value
corresponds to a bias standard deviation of \(0.05\). We use
\begin{equation}
n\in\{256,362,512,724,1024,1448,2048\},
\qquad L\in\{1,2,3,4,5\}.
\end{equation}
The controlled comparisons use the kernel-level regularization
\(\gamma=0.1\), while \(\gamma=5\times10^{-3}\) is reported only as a stress
test. The \(N\)--\(n\) diagnostic in Subsection~\ref{sec:joint-scaling} uses
\(\gamma=10^{-2}\) and \(0.2\).

Directly estimating the post-correction residual from independent Monte Carlo
blocks is inefficient because the ensemble fluctuation of the observable is
larger than the residual being measured. We therefore use a paired estimator.
For \(M=6000\) feature draws, let \(\widehat G\) and the analogous test objects
denote the sample means of the corresponding kernel objects over those draws,
and write the fluctuation of draw \(m\) as \(\delta_m\). For any of the three
error observables, we evaluate the exact value together with its first- and
second-order terms around the same sample mean:
\begin{equation}
r_m
=\mathcal E_m-\mathcal E^{(0)}
-D\mathcal E[\delta_m]
-\frac12D^2\mathcal E[\delta_m,\delta_m].
\label{eq:paired-residual}
\end{equation}
Because \(\sum_m\delta_m=0\), the sample mean of the linear term vanishes
identically. The sample mean of the quadratic term is the plug-in covariance
correction, and \(\overline r\) estimates the remainder after that correction.
We define
\begin{equation}
R_0=\left|\overline{\mathcal E}-\mathcal E^{(0)}\right|,
\qquad
R_1=\left|\overline{\mathcal E}-\mathcal E^{(0)}
-\mathcal E^{(\mathrm{cov})}\right|.
\label{eq:R0R1}
\end{equation}

All displayed bands are pointwise \(95\%\) percentile intervals from \(2000\)
bootstrap resamples of the feature draws. The bootstrap retains the paired
structure and treats the sample-mean expansion point as fixed. It therefore
does not include the \(O(M^{-1})\) bias caused by estimating that point.
Likewise, the finite test grid approximates the population integral, and its
integration error is not included in the bands. The design, feature, and
bootstrap seeds are \(1234\), \(7777\), and \(20240\), respectively, and all
linear algebra uses float64 with adaptive jitter at \(10^{-12}\). Six internal
identities---including the matrix and contracted forms of the covariance terms
and the spectral form of the training error---agree to within
\(5.7\times10^{-15}\) in the reported self-checks.

\subsection{Direct width scaling of the fluctuation vertex}
\label{sec:vertex-scaling}

Figure~\ref{fig:loop-scaling-controlled} shows the covariance corrections
themselves at \(\gamma=0.1\). Over the seven widths at \(L=2\), their fitted
width exponents are \(-1.012\), \(-1.017\), and \(-1.017\) for training error,
test error, and the absolute conditional gap.
Because the formulas already contain an explicit \(1/n\), this observation
alone would only show that the remaining contractions vary slowly.

We therefore also measure the kernel fluctuation before inserting it into any
error formula. Across the four bias variances in
Figure~\ref{fig:bias-vertex}, the fitted slopes of
\(n\,\mathbb E[\|\Delta_K\|_F^2]\) lie between \(-0.017\) and \(-0.006\), with
mean \(-0.013\). This is a direct numerical check that the empirical-kernel
covariance vertex remains \(O(1)\) on the displayed width grid. The depth
panels also separate two effects that the correction formula combines: the
measured vertex norm and the amplification by the mean resolvent.

\begin{figure}[t]
    \centering
    \includegraphics[width=\linewidth]{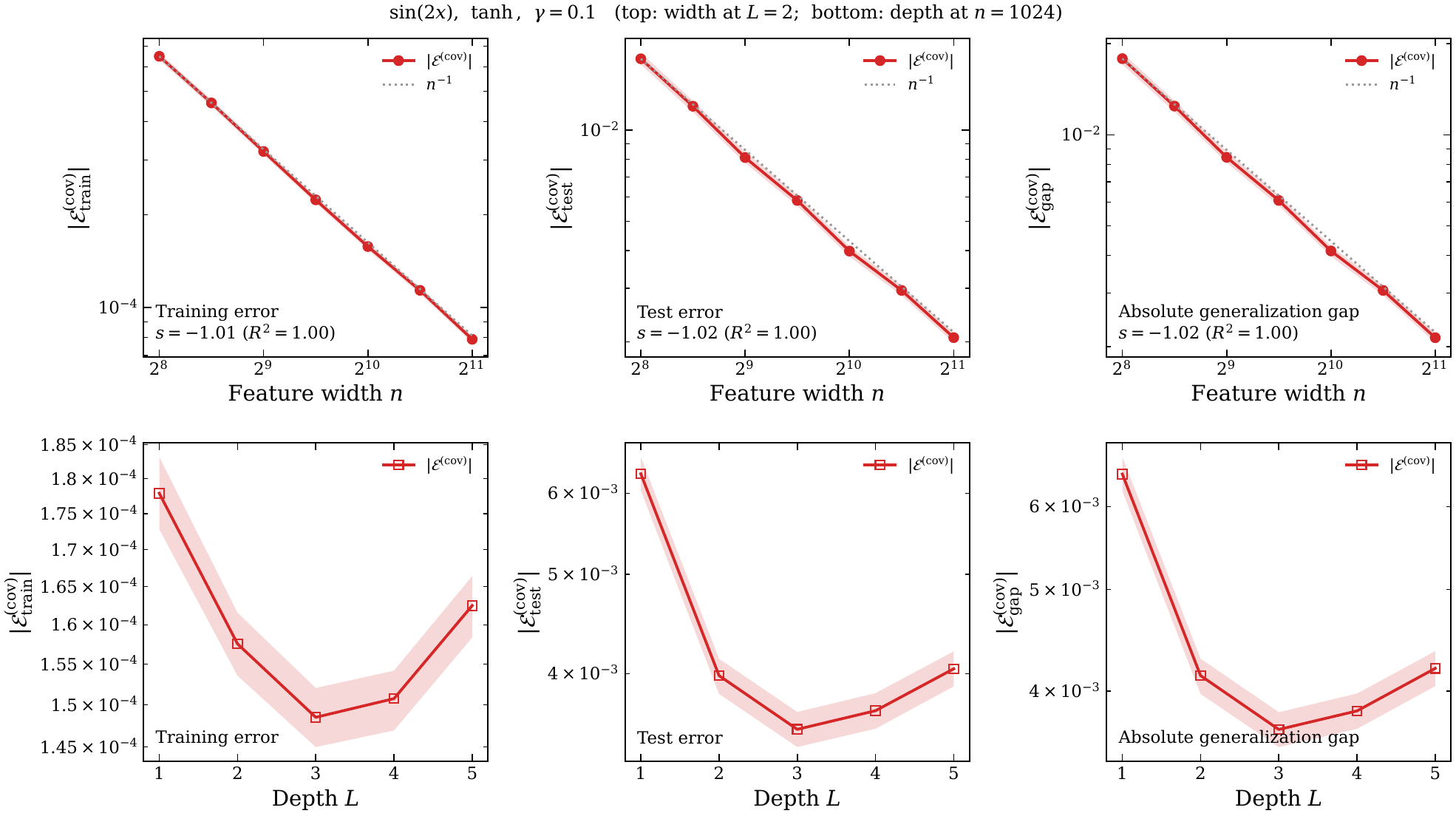}
    \caption{Width and depth dependence of the covariance corrections for the
    \(\sin(2x)\) target with \(\tanh\) features at the controlled value
    \(\gamma=0.1\). Bands are pointwise paired-bootstrap intervals. The
    near-\(-1\) width slopes are interpreted together with the direct vertex
    measurement in Figure~\ref{fig:bias-vertex}.}
    \label{fig:loop-scaling-controlled}
\end{figure}

\begin{figure}[t]
    \centering
    \includegraphics[width=\linewidth]{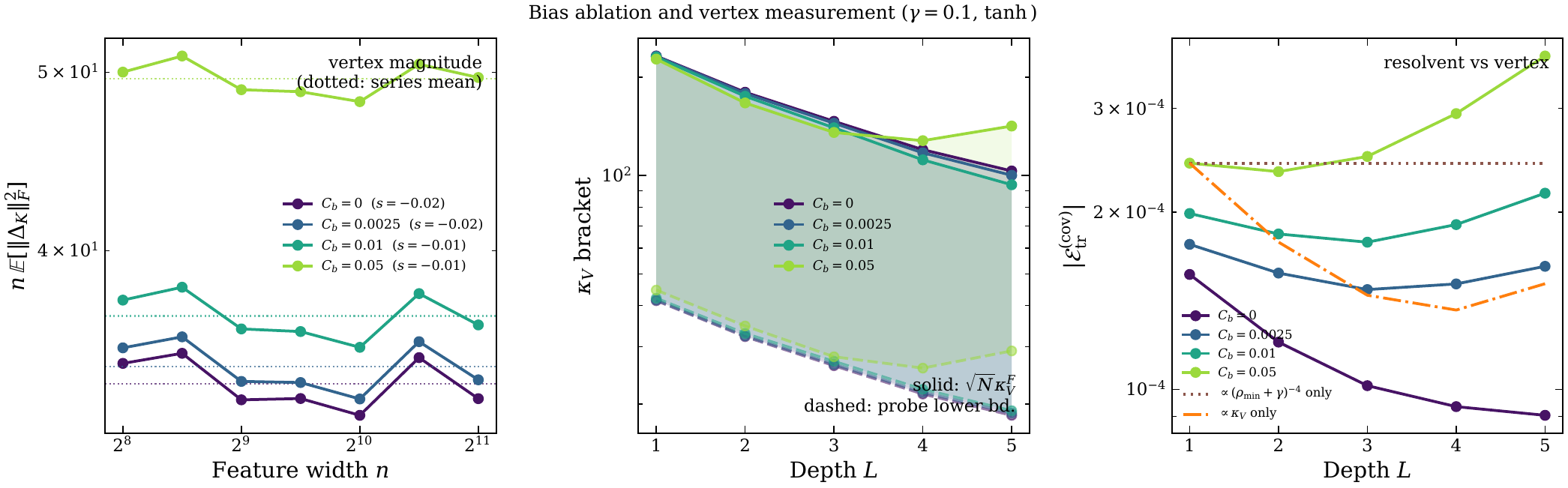}
    \caption{Bias ablation and direct vertex diagnostics at \(\gamma=0.1\).
    Left: \(n\mathbb E[\|\Delta_K\|_F^2]\) is approximately width-independent.
    Middle: bounds on the contraction-map norm as depth varies. Right: the
    training correction reflects both resolvent and vertex effects, so depth
    dependence cannot be assigned to either factor alone.}
    \label{fig:bias-vertex}
\end{figure}

\subsection{Removal of the leading finite-width discrepancy}
\label{sec:residual-hierarchy}

Figure~\ref{fig:representative-width-sweep} gives a representative controlled
comparison. The mean-object prediction misses the ensemble average by a
visible \(O(n^{-1})\) term, while the mean-plus-covariance prediction is
indistinguishable from it at the scale of the upper panels. The lower panels
show the same comparison through \(R_0\) and \(R_1\).

This behavior is not confined to the representative target. At
\(\gamma=0.1\), we test six target--architecture settings, seven widths, and
three observables, for a total of \(126\) pointwise comparisons. In all \(126\)
cases, the reported upper confidence bound for \(R_1\) is below the lower
confidence bound for \(R_0\). The conservative ratio of these two bounds has
minimum \(10.7\) and median \(72.4\) over the \(126\) comparisons. In \(95\)
of them the signed \(R_1\) interval contains zero, so ratios based only on the
point estimate would overstate the resolved accuracy; we quote the bound-based
ratios instead.

The scaling of the residuals provides a less circular test of the hierarchy
than the ratios themselves.
Across \(\gamma\in\{0.05,0.1,0.2\}\), the fitted \(R_0\) slopes for all three
observables lie between \(-1.018\) and \(-1.013\), with \(R^2>0.998\).
After the covariance term is removed, the central \(R_1\) slopes range from
\(-2.55\) to \(-1.59\). Their bootstrap intervals are broad because \(R_1\)
approaches the Monte Carlo floor, but they are compatible with the
\(O(n^{-3/2})\) remainder bound and are inconsistent with an unremoved
leading \(n^{-1}\) term. We therefore interpret Figure~\ref{fig:residual-scaling}
as evidence that the covariance term removes the leading discrepancy, not as
a precise measurement of the next exponent.

\begin{figure}[t]
    \centering
    \includegraphics[width=\linewidth]{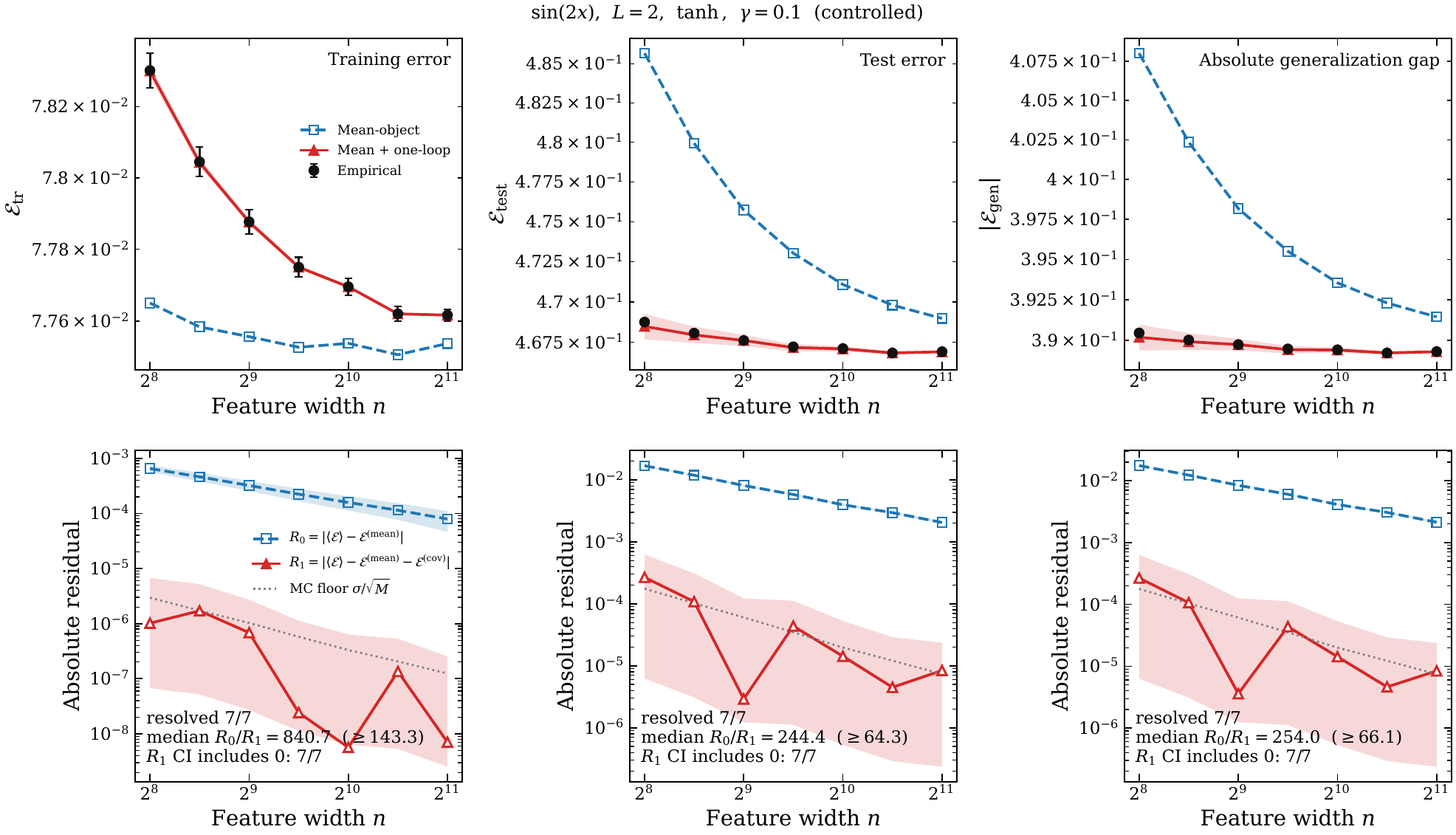}
    \caption{Representative width sweep for \(y(x)=\sin(2x)\), \(L=2\),
    \(\tanh\), and \(\gamma=0.1\). Top: empirical ensemble means, mean-object
    predictions, and mean-plus-covariance predictions. Bottom: the residuals
    \(R_0\) and \(R_1\) from \eqref{eq:R0R1}. The gap panels display the
    magnitude of the signed ensemble-mean gap.}
    \label{fig:representative-width-sweep}
\end{figure}

\begin{figure}[t]
    \centering
    \includegraphics[width=\linewidth]{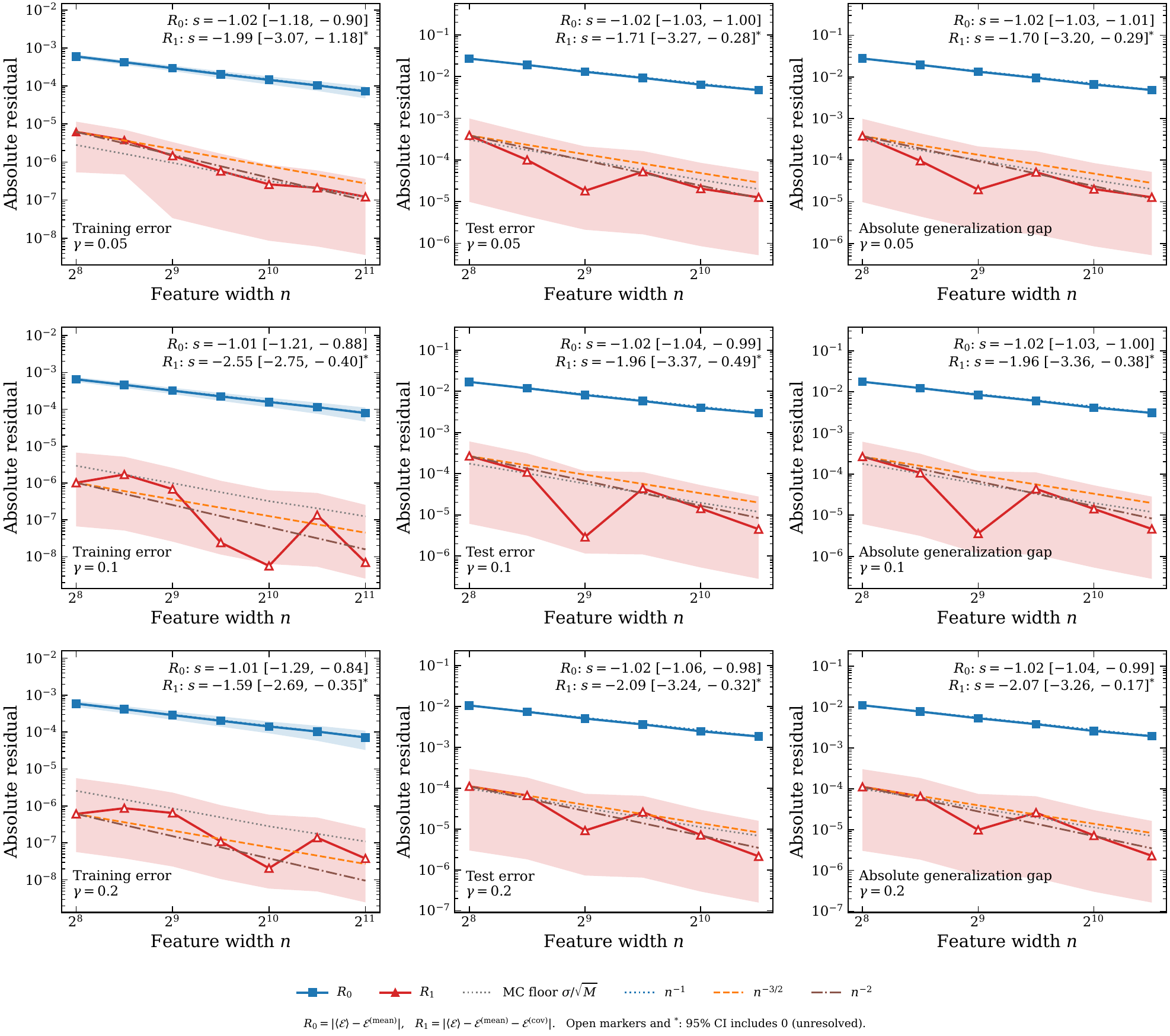}
    \caption{Residual scaling for \(\gamma=0.05,0.1,0.2\). The mean-object
    residual \(R_0\) follows \(n^{-1}\). The post-correction residual \(R_1\)
    is much smaller and approaches the Monte Carlo floor; the printed
    bootstrap intervals do not support distinguishing \(n^{-3/2}\) from
    \(n^{-2}\) uniformly across observables.}
    \label{fig:residual-scaling}
\end{figure}

\subsection{Mixed train--test covariance is an observable-level effect}
\label{sec:mixed-ablation}

The following ablation turns the formal distinction between training and test
error into a qualitative one. We keep the training design and feature draws fixed but
evaluate test error under three populations:
\(\mathcal N(0,1)\), \(\mathcal N(1,1)\), and \(\mathcal N(0,2^2)\).
Consequently, the training correction is identical in the three experiments,
up to numerical precision, whereas the mixed train--test covariance maps
change.

Figure~\ref{fig:mixed-ablation-test} compares the full correction with an
ablation that retains only train--train kernel covariance. The train-restricted
approximation moves the test prediction in the wrong direction for all three
populations, while the full correction tracks the empirical mean. Over
\(n\in\{256,512,1024,2048\}\), the magnitude of the mixed component is
\(2.56\)--\(4.09\) times the magnitude of the final test correction. The ratio
exceeds one because the mixed and train-restricted contributions partially
cancel; it should not be read as a positive fractional decomposition. This
ablation is the numerical counterpart of
Proposition~\ref{prop:nonidentifiability}: identical train-restricted data do
not determine the test correction.

\begin{figure}[t]
    \centering
    \includegraphics[width=\linewidth]{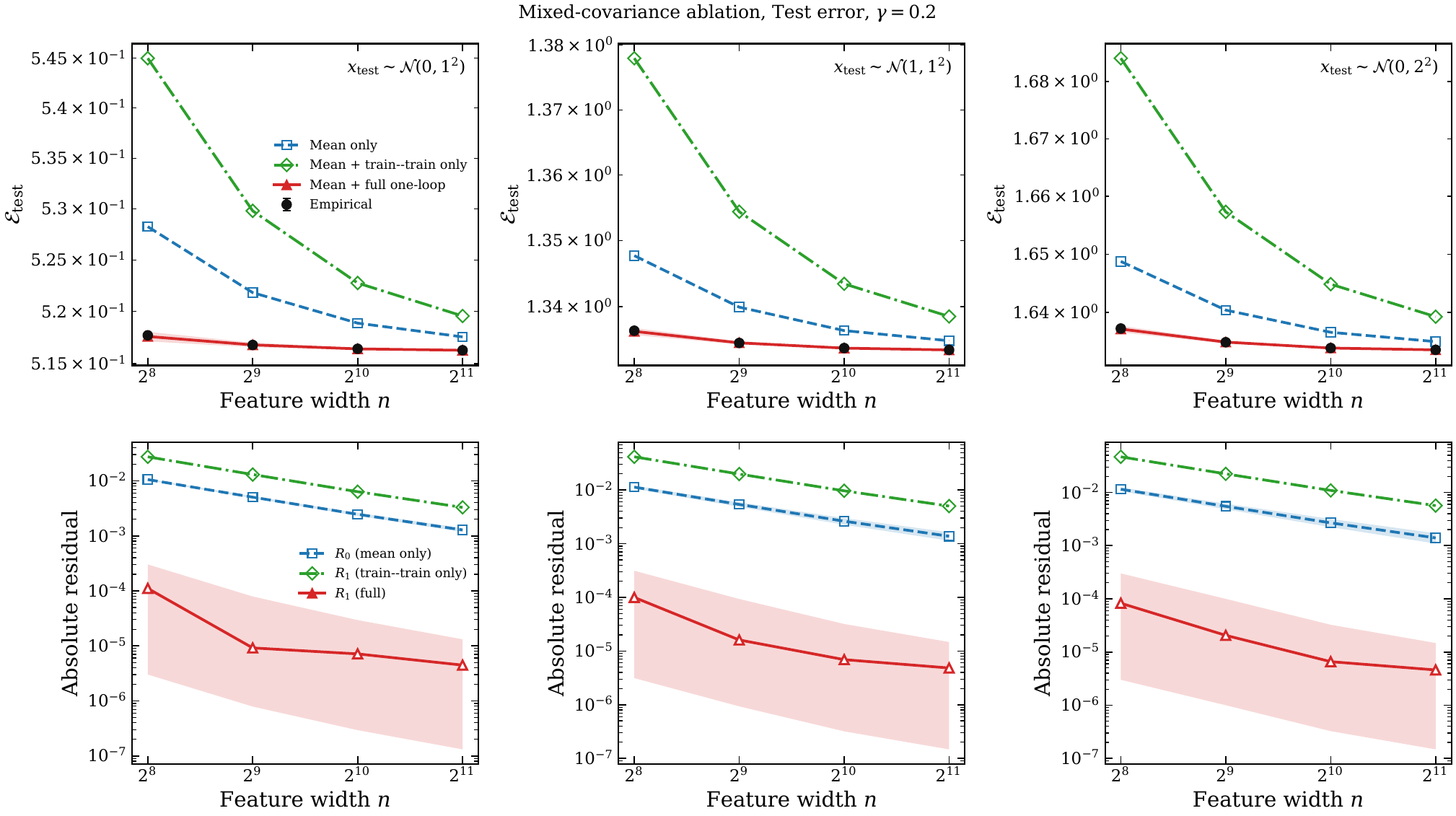}
    \caption{Mixed-covariance ablation for test error at \(\gamma=0.2\).
    Columns correspond to three test populations with an identical training
    design. Retaining only train--train covariance leaves a residual \(2.6\)--\(4.1\)
    times the mean-object baseline residual; including the mixed train--test
    tensors removes that discrepancy.}
    \label{fig:mixed-ablation-test}
\end{figure}

\subsection{A width--regularization boundary for perturbative accuracy}
\label{sec:validity-map}

The finite resolvent identity remains exact without a Neumann condition, but a
second-order truncation need not be accurate when
\begin{equation}
q_\Theta=\left\|(G+\gamma I)^{-1}\Delta_K\right\|_2
\end{equation}
is of order one. Figure~\ref{fig:validity-map} measures its mean, \(90\)th
percentile, and exceedance probability on
\(n\in\{256,512,1024,2048\}\) and
\(\gamma\in\{10^{-3},3\times10^{-3},10^{-2},3\times10^{-2},0.1,0.3,1\}\).
At the smallest regularization, the maximum observed
\(\Pr(q_\Theta\geq1)\) is \(1\); at \(\gamma=1\) it is \(0\). Increasing width
shifts the crossing to smaller \(\gamma\), as expected from shrinking kernel
fluctuations.

More importantly, \(q_\Theta\) predicts truncation error rather than merely
kernel variability. Across the full phase grid, the Spearman correlations
between \(\mathbb E [q_\Theta]\) and the relative post-correction errors are
\(0.82\), \(0.95\), and \(0.95\) for training error, test error, and the
conditional gap. This identifies the role of regularization: finite width can
be perturbative at moderate \(n\) when the resolvent is controlled, and
nonperturbative at larger \(n\) when small-eigenvalue modes are strongly
amplified. The value \(\gamma=0.1\) used in the main comparisons lies on the
controlled side of the displayed boundary; \(\gamma=5\times10^{-3}\) does not
lie uniformly on that side.

\begin{figure}[t]
    \centering
    \includegraphics[width=\linewidth]{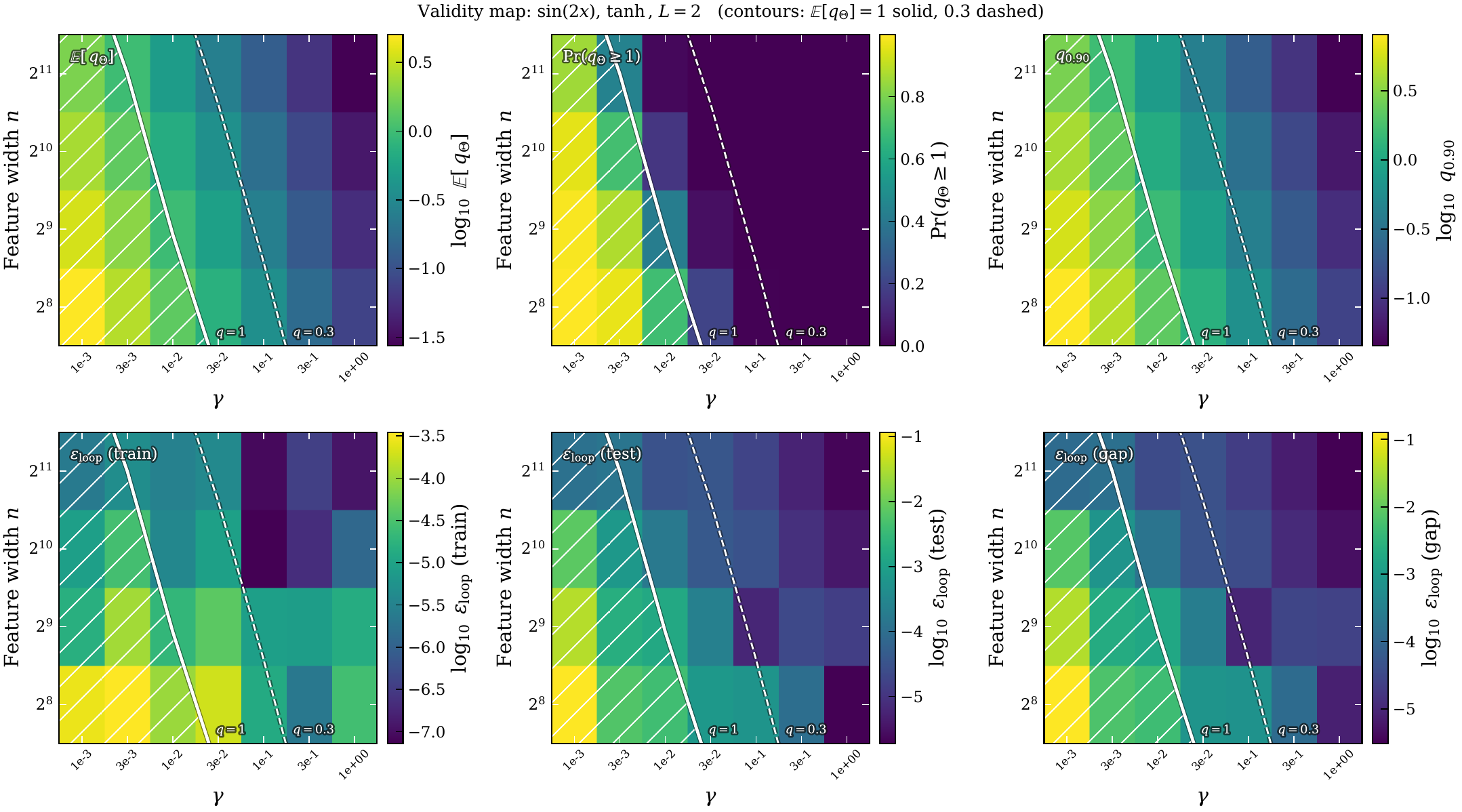}
    \caption{Validity map for \(y(x)=\sin(2x)\), \(L=2\), and \(\tanh\)
    features. The upper row reports the mean, exceedance probability, and
    \(90\)th percentile of \(q_\Theta\); the lower row reports the relative
    post-correction errors. The solid and dashed white contours mark
    \(\mathbb E[q_\Theta]=1\) and \(\mathbb E[q_\Theta]=0.3\); hatching marks
    the region above \(\mathbb E[q_\Theta]=1\).}
    \label{fig:validity-map}
\end{figure}

\subsection{Bias, sample size, and reporting limits}
\label{sec:joint-scaling}

The bias ablation quantifies an architectural dependence that is invisible if
the finite-width vertex is treated as a universal constant. At \(n=1024\),
the main bias variance \(C_b=2.5\times10^{-3}\) increases the magnitude of the
training correction by a factor \(1.31\) relative to \(C_b=0\), and the factor
stays between \(1.28\) and \(1.32\) across the width grid. Larger bias
variances change both the empirical-kernel covariance and its contraction with
the resolvent, as shown in Figure~\ref{fig:bias-vertex}. This dependence is an
initialization effect within the frozen-feature ensemble, not feature learning.

Finally, a descriptive \(N\)--\(n\) fit confirms that the mean-object training
error is essentially width-independent while the covariance correction retains
its inverse-width factor. The fitted width exponents of the correction are
\(-1.001\) at \(\gamma=10^{-2}\) and \(-1.007\) at \(\gamma=0.2\). The fitted
sample-size exponent changes from \(-0.66\) to \(+0.41\), respectively, showing
that a single joint \(N,n\) power law is not supported across regularization
regimes. Because the theory holds \(N\) fixed and only four sample sizes are
used, these \(N\)-exponents remain descriptive. Additional residual, mixed-gap,
stress-test, \(N\)--\(n\), distributional, and spectral diagnostics are reported
in Appendix~\ref{app:numerics}.

The study remains conditional on one design per setting and uses one-dimensional
synthetic targets. The paired bootstrap does not include design-to-design
variation or finite test-integration error. These limitations prevent a claim
of uniform accuracy in sample size, depth, or vanishing regularization. The
executed notebook, numerical seeds, summary file, and the per-experiment records
behind every reported number are included with the source archive.

\section{Conclusion}

We derived centered covariance corrections to the feature-ensemble-averaged
training error, test error, and conditional generalization gap of fixed-design
random feature ridge regression. The effective vertex in these formulas is the
covariance of the empirical feature kernel. It is not the connected
preactivation vertex alone; the within-feature Wick terms in
\eqref{eq:empirical-connected-relation} remain even when the latter vanishes.
The baseline is the exact finite-width mean object, so the result isolates the
centered fluctuation correction rather than claiming a complete expansion of
all mean quantities around their infinite-width limits.

The observable-level analysis yields two conclusions beyond the generic fact
that wide-network statistics possess finite-width corrections. First, the test
error and conditional gap depend on mixed train--test covariance maps. They are
not universal functionals of the training-kernel law, as formalized by
Proposition~\ref{prop:nonidentifiability} and demonstrated by the population-shift
ablation. Second, the accuracy of the covariance truncation is governed by a
resolvent-amplified fluctuation, not by width alone. The measured
width--regularization boundary and its correlation with post-correction error
make this distinction explicit.

In the controlled regime, the covariance term removes the leading
\(O(n^{-1})\) discrepancy in all reported target--architecture settings. The
remaining residual approaches the Monte Carlo floor with exponents compatible
with the theoretical remainder bound. The direct measurement of
\(n\mathbb E[\|\Delta_K\|_F^2]\), the bias ablation, and the mixed-covariance
ablation distinguish kernel statistics, resolvent amplification, and
off-training geometry rather than combining them into a single accuracy plot.

These conclusions remain conditional on a fixed design and frozen hidden
features. They do not address feature learning, NTK evolution, or the
performance gap between trained networks and their kernel limits. The most
direct extension is therefore not another resolvent order in the same frozen
model, but an observable-level treatment in which the train--test kernel
objects evolve during training. A second direction is to derive the mixed
vertices analytically for specific architectures and to test the resulting
predictions across design draws and higher-dimensional data.

\subsubsection*{Acknowledgments}
Taeyoung Kim is supported by a KIAS Individual Grant (AP102201) at Korea Institute for Advanced Study and supported by the Center for Advanced Computation at Korea Institute for Advanced Study.

\newpage
\appendix
\section{Additional Numerical Evidence}
\label{app:numerics}

This appendix reports the diagnostics that support, but are not needed to
state, the three main numerical conclusions. All settings and uncertainty
conventions are those of Subsection~\ref{sec:numerical-setup}.

\subsection{Cross-setting accuracy and the paired noise floor}

Figure~\ref{fig:improvement-summary} summarizes the point-estimate improvement
\(R_0/R_1\) across the six target--architecture settings and the controlled
regularization ladder. Ratios become unstable when \(R_1\) is not resolved;
the numerical claims in the main text therefore use confidence-bound ratios,
not the bars in this figure.

Figure~\ref{fig:residual-exponents-noise} shows the bootstrap intervals for the
residual exponents and compares the paired estimator with a split-ensemble
estimator. The split residual remains at its sampling floor over much of the
width grid, whereas subtracting the realization-wise linear and quadratic
terms exposes a much smaller residual without changing its sample mean.

\begin{figure}[H]
    \centering
    \includegraphics[width=\linewidth]{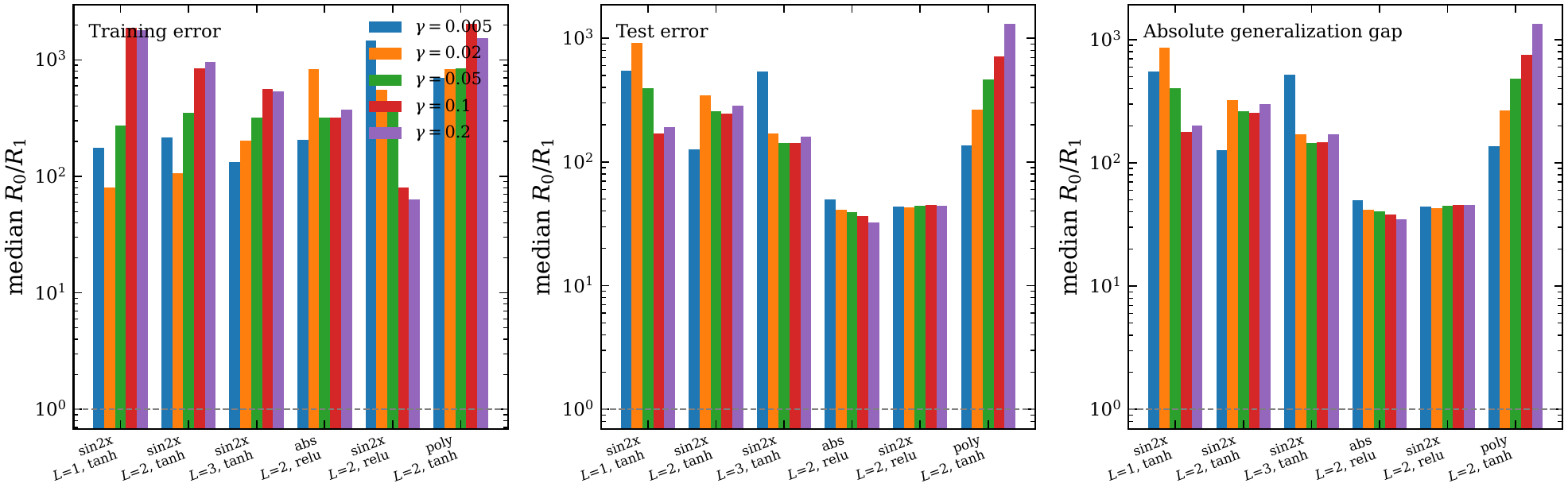}
    \caption{Median point-estimate ratios \(R_0/R_1\) across the target and
    architecture settings. These bars are descriptive; the main text reports
    ratios formed from the absolute confidence bounds.}
    \label{fig:improvement-summary}
\end{figure}

\begin{figure}[H]
    \centering
    \includegraphics[width=\linewidth]{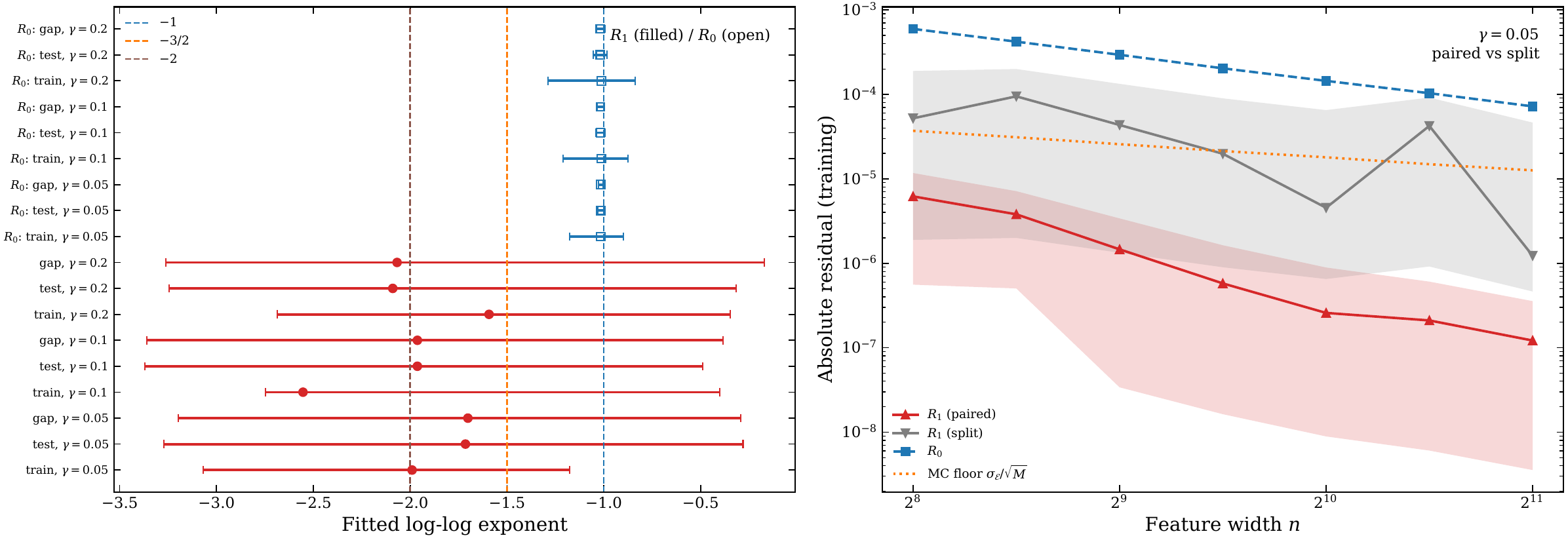}
    \caption{Left: fitted exponents for \(R_0\) and \(R_1\), with bootstrap
    intervals. Right: paired and split estimates of the post-correction
    training residual at \(\gamma=0.05\), together with the Monte Carlo floor.}
    \label{fig:residual-exponents-noise}
\end{figure}

\clearpage
\subsection{Conditional gap ablation and the resolvent diagnostic}

The conditional-gap ablation in Figure~\ref{fig:mixed-ablation-gap} mirrors the
test-error result in Figure~\ref{fig:mixed-ablation-test}. The training
correction is common to the three columns, so the failure of the
train-restricted approximation is inherited from test error. The lower row
shows that the full mixed correction removes the leading discrepancy.

Figure~\ref{fig:eps-vs-q} displays the pointwise relation between the
resolvent diagnostic and relative post-correction error that underlies the
rank correlations reported in Subsection~\ref{sec:validity-map}. The relation
is not a deterministic bound: architecture, target, and individual covariance
contractions still affect the vertical spread.

\begin{figure}[H]
    \centering
    \includegraphics[width=\linewidth]{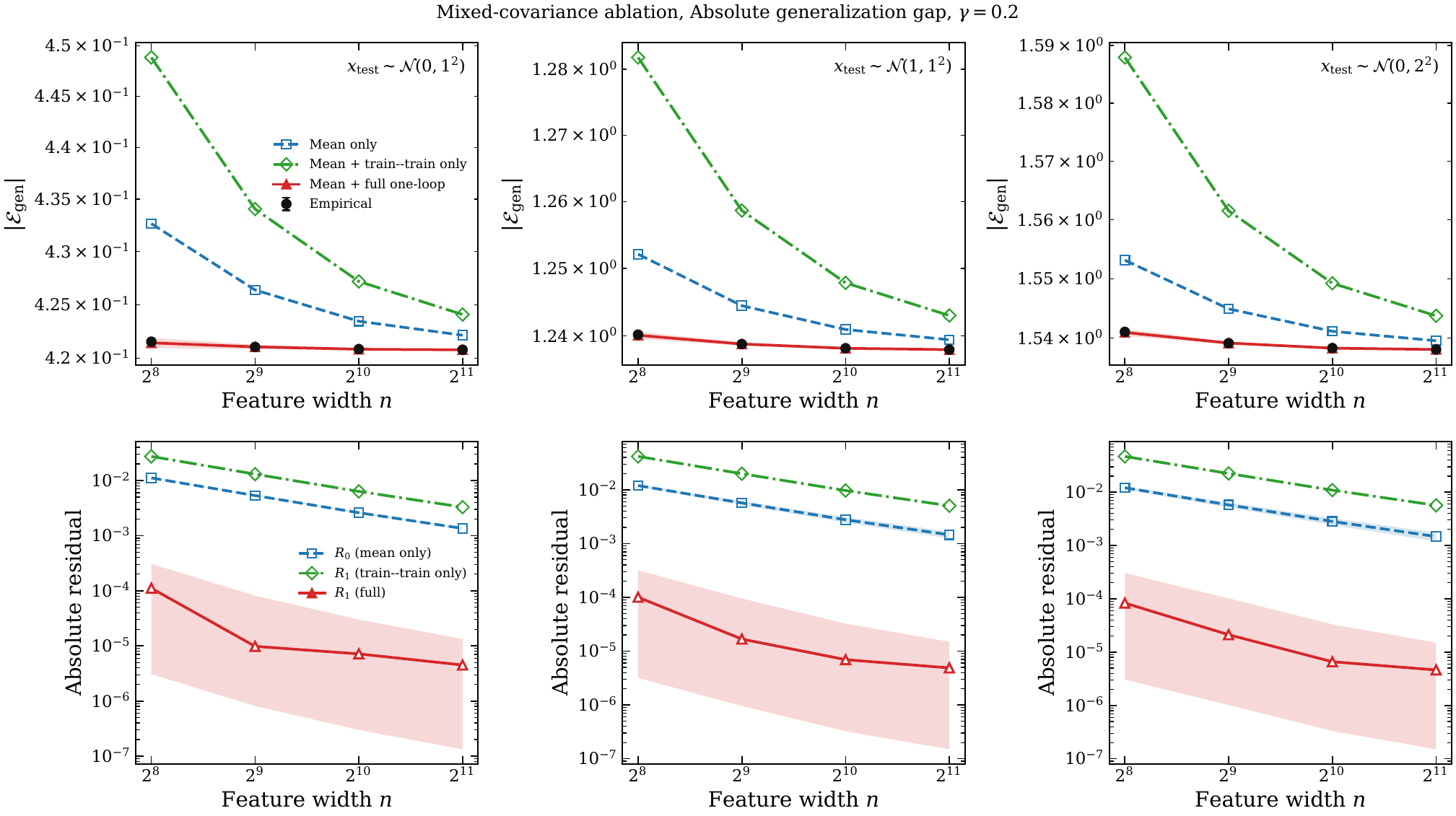}
    \caption{Mixed-covariance ablation for the absolute conditional
    generalization gap at \(\gamma=0.2\).}
    \label{fig:mixed-ablation-gap}
\end{figure}

\begin{figure}[H]
    \centering
    \includegraphics[width=0.88\linewidth]{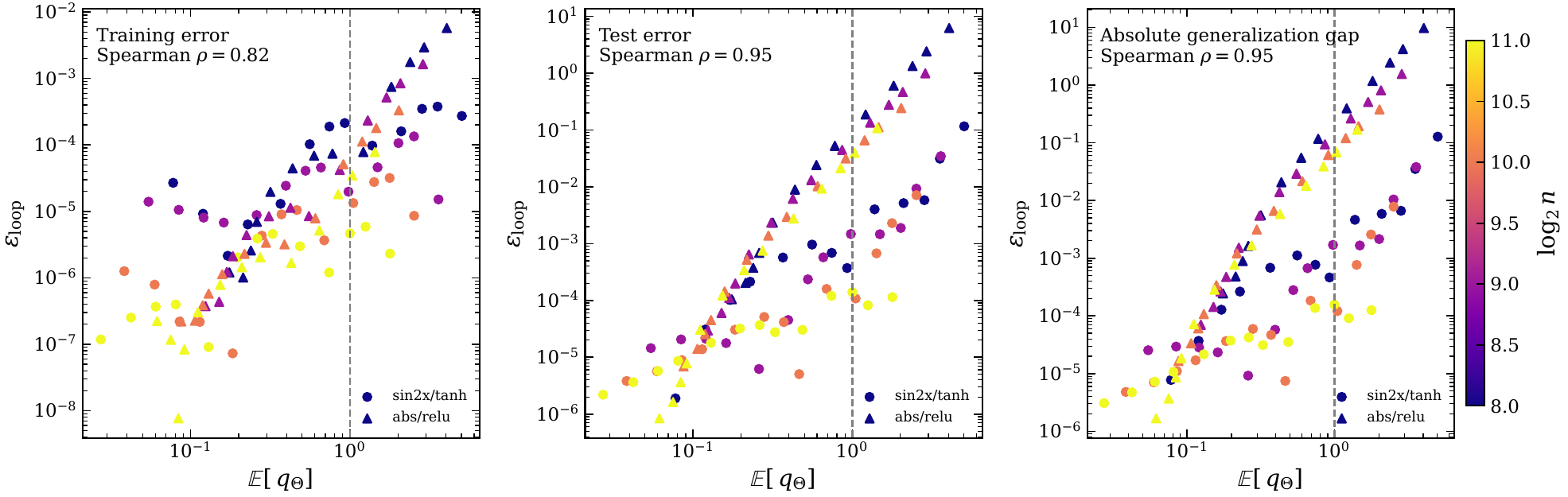}
    \caption{Relative post-correction error versus
    \(q_\Theta=\|(G+\gamma I)^{-1}\Delta_K\|_2\) across widths,
    regularization values, targets, and activations.}
    \label{fig:eps-vs-q}
\end{figure}

\clearpage
\subsection{ReLU--\texorpdfstring{\(\lvert x\rvert\)}{|x|} stress test}
\label{app:relu-stress}

At weak regularization, the ReLU--\(\lvert x\rvert\) setting produces an
unstable-looking truncated prediction. The larger paired ensemble and the
regularization ladder separate four possible causes: cancellation among large
terms, estimator variance, mixed-covariance growth, and resolvent amplification.
At \(\gamma=5\times10^{-3}\), the median relative test residual over the four
widths is \(8.3\times10^{-2}\), and the largest cancellation ratio is \(3.2\).
At \(\gamma=0.5\), the median relative test residual falls to
\(3.71\times10^{-5}\). This behavior is therefore primarily a stress-regime
effect rather than evidence that nonanalytic targets invalidate the
empirical-kernel vertex.

Figure~\ref{fig:relu-failure-diagnosis} summarizes these diagnostics.
Figure~\ref{fig:relu-term-decomposition} then resolves the test covariance
correction into its five terms. Individual terms remain smooth in width; their
sum is sensitive to cancellation at small regularization.

\begin{figure}[H]
    \centering
    \includegraphics[width=\linewidth]{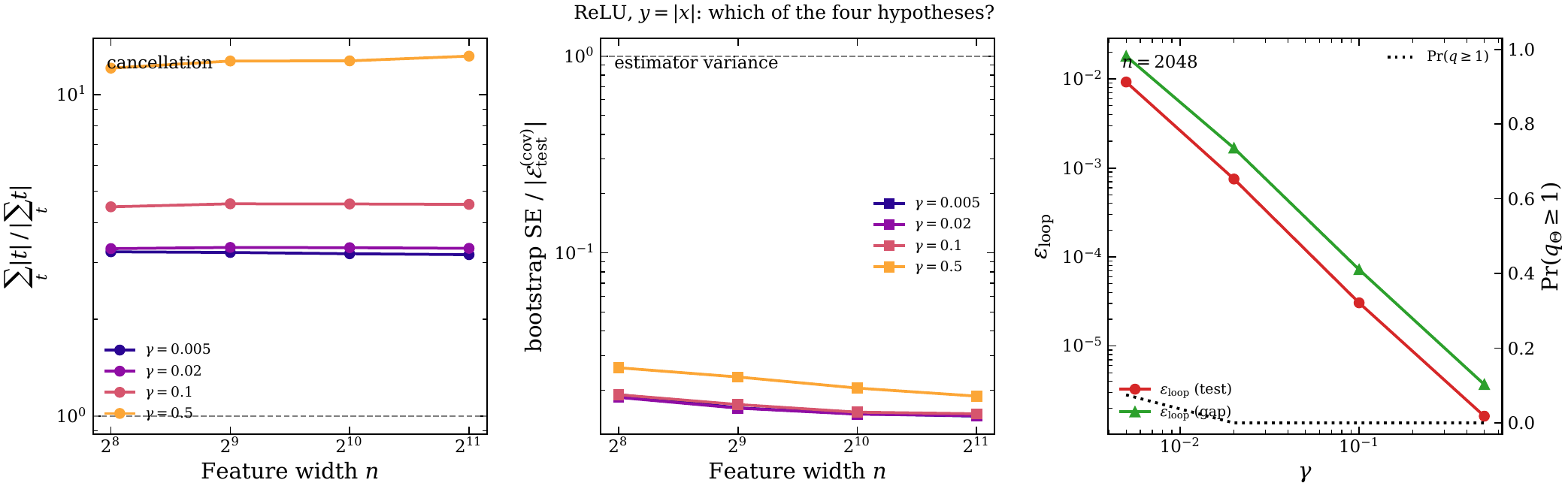}
    \caption{Diagnostics for the ReLU--\(\lvert x\rvert\) stress test:
    cancellation, bootstrap uncertainty, and the joint effect of regularization
    on \(q_\Theta\) and truncation error.}
    \label{fig:relu-failure-diagnosis}
\end{figure}

\begin{figure}[H]
    \centering
    \includegraphics[width=\linewidth]{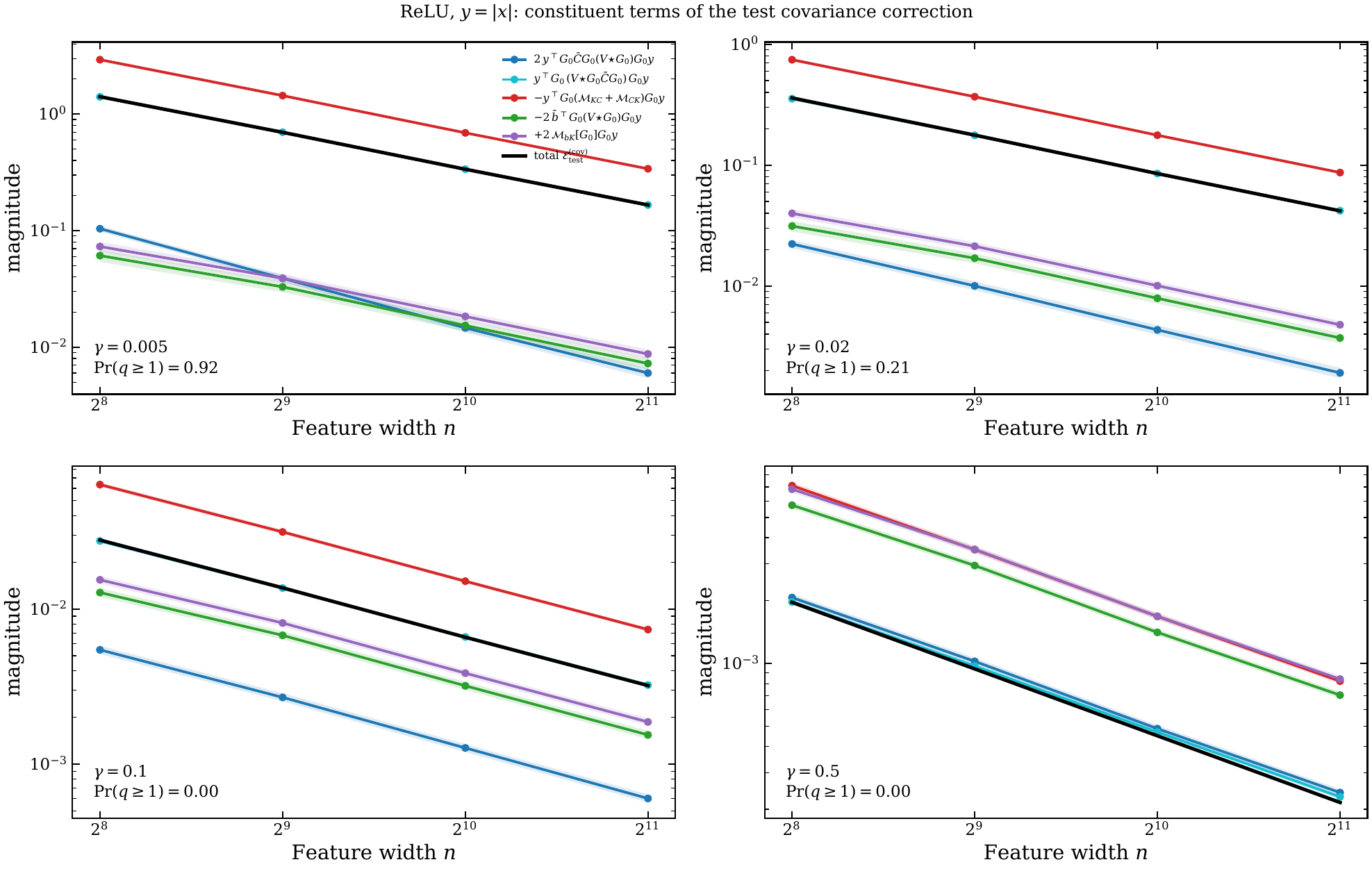}
    \caption{Term-by-term decomposition of the ReLU--\(\lvert x\rvert\) test
    covariance correction over the regularization ladder. Bands are
    paired-bootstrap intervals.}
    \label{fig:relu-term-decomposition}
\end{figure}

\clearpage
\subsection{Descriptive joint \texorpdfstring{\(N\)--\(n\)}{N--n} scaling}

Figure~\ref{fig:Nn-exponents} compares the fitted sample-size and width
exponents. The mean-object prediction is nearly independent of \(n\), while
the covariance term retains an exponent near \(-1\) at both regularization
values. Its dependence on \(N\) changes sign, consistent with the fact that
the fixed-\(N\) theorem does not imply a joint scaling law.

The heat maps in Figures~\ref{fig:Nn-small} and \ref{fig:Nn-large} show the
underlying training-error values and effective-dimension diagnostics. The
target normalization is common to all \(N\), so the remaining \(N\)-dependence
is not caused by rescaling the labels with training-sample statistics.

\begin{figure}[H]
    \centering
    \includegraphics[width=0.72\linewidth]{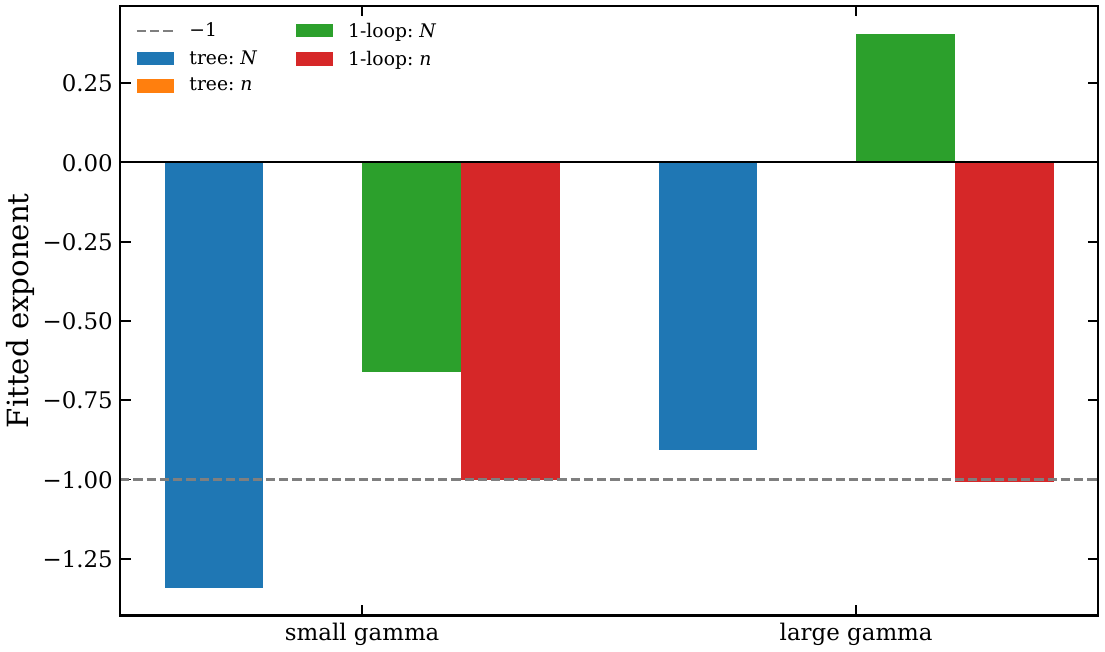}
    \caption{Fitted exponents in the descriptive model
    \(\log|\mathcal E|=c+a\log N+b\log n\) at
    \(\gamma=10^{-2}\) and \(0.2\).}
    \label{fig:Nn-exponents}
\end{figure}

\begin{figure}[H]
    \centering
    \includegraphics[width=\linewidth]{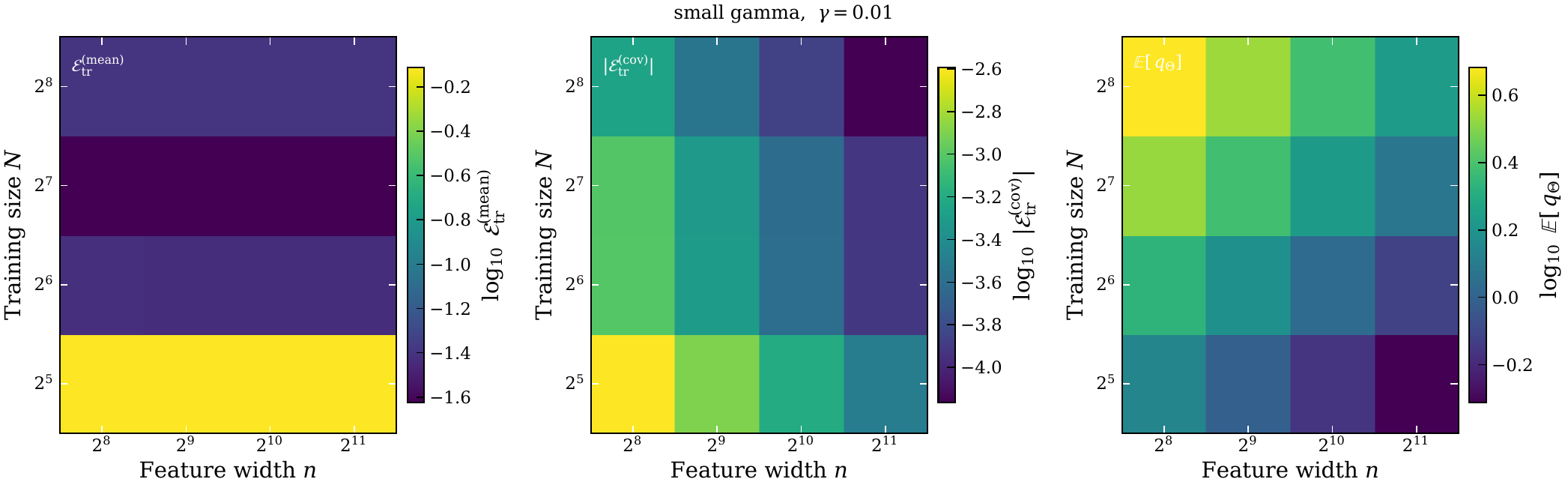}
    \caption{Training mean-object and covariance terms across \(N\) and \(n\)
    at \(\gamma=10^{-2}\), with effective-dimension diagnostics.}
    \label{fig:Nn-small}
\end{figure}

\begin{figure}[H]
    \centering
    \includegraphics[width=\linewidth]{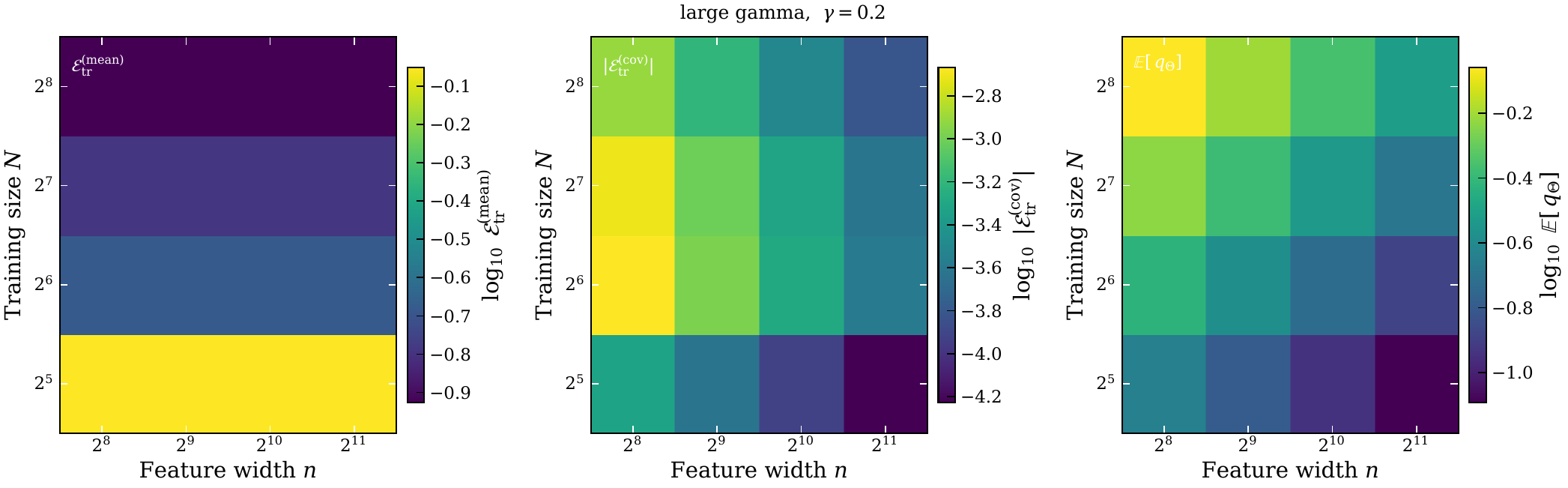}
    \caption{Training mean-object and covariance terms across \(N\) and \(n\)
    at \(\gamma=0.2\), with effective-dimension diagnostics.}
    \label{fig:Nn-large}
\end{figure}

\clearpage
\subsection{Distributional and spectral views}

Figure~\ref{fig:ensemble-distributions} separates two notions that are easily
conflated. The covariance correction predicts the displacement of the
ensemble mean relative to the mean-object baseline; it does not predict the
width of the realization-to-realization distribution. The histograms narrow
with \(n\), while the corrected mean continues to track the empirical mean.

Figure~\ref{fig:spectral-modes} decomposes the training mean-object and
covariance terms in the eigenbasis of \(G\). The dominant contributions occur
in modes whose eigenvalues are comparable to or larger than the regularization
scale. This provides a mode-resolved view of the amplification summarized by
\(q_\Theta\).

\begin{figure}[H]
    \centering
    \includegraphics[width=\linewidth]{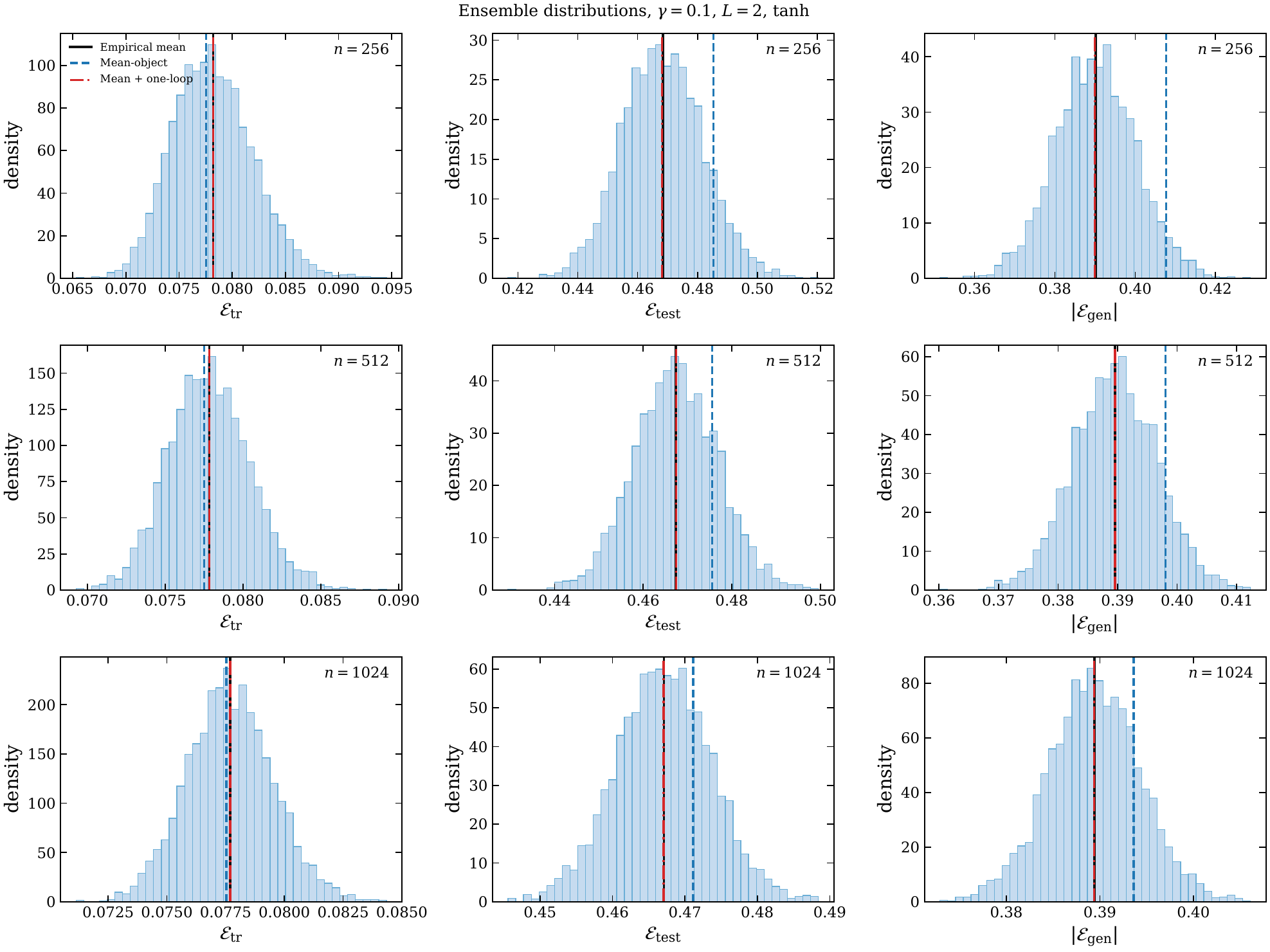}
    \caption{Ensemble distributions of training error, test error, and the
    absolute conditional gap for three widths at \(\gamma=0.1\). Vertical
    lines mark the empirical mean, mean-object prediction, and
    mean-plus-covariance prediction.}
    \label{fig:ensemble-distributions}
\end{figure}

\begin{figure}[H]
    \centering
    \includegraphics[width=\linewidth]{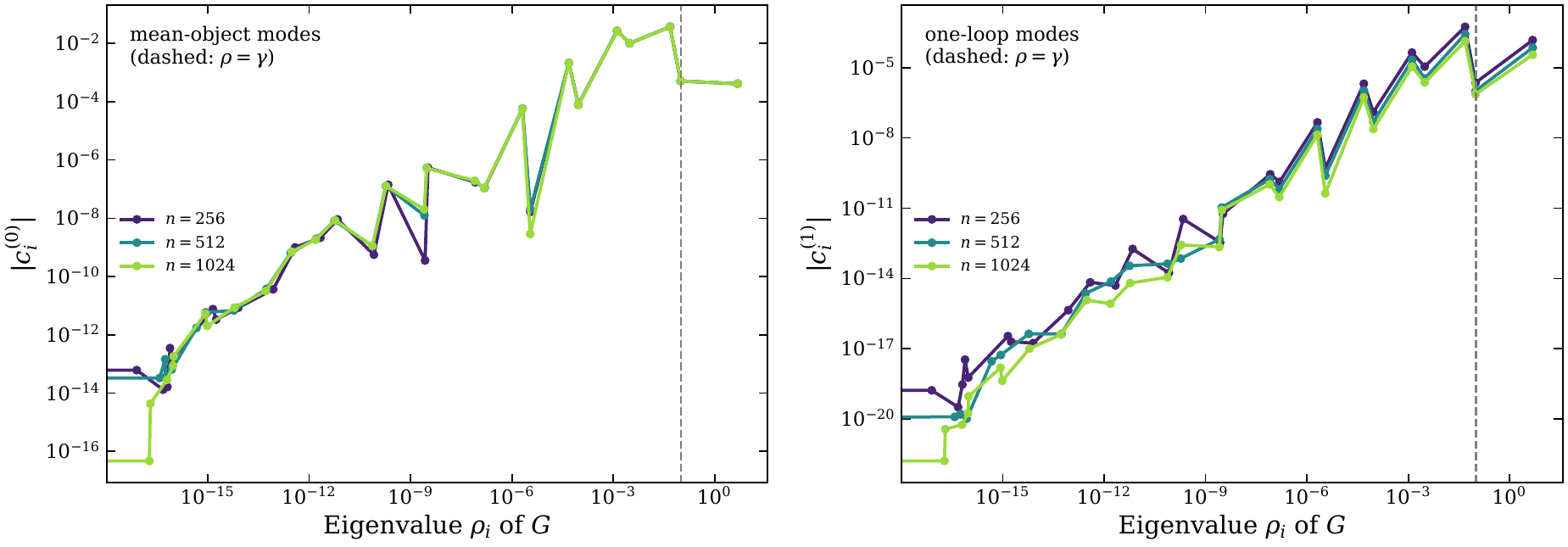}
    \caption{Mode-resolved mean-object and covariance contributions to the
    training error in the eigenbasis of \(G\). The dashed line marks
    \(\rho=\gamma\).}
    \label{fig:spectral-modes}
\end{figure}

\newpage
\bibliographystyle{MLSTEFTNO}
\bibliography{MLSTEFTNO}
\end{document}